\documentclass[preprint,12pt, 3p]{elsarticle}
\usepackage{graphicx}
\usepackage{latexsym}
\usepackage[utf8]{inputenc}
\usepackage{amsthm}
\usepackage[ruled,algonl,vlined]{algorithm2e}

\usepackage{mathtools}  
\usepackage{amsmath}
\usepackage{amssymb}
\usepackage{multirow}
\usepackage{url}
\usepackage{float}
\usepackage{xcolor}
\usepackage{booktabs}
\usepackage{fixltx2e}
\usepackage{adjustbox}
\usepackage{soul}

\usepackage{lineno}

\newcommand{\VK}[1]{\textcolor{black}{#1}}

\usepackage{caption,subcaption}
\usepackage{setspace}
\usepackage{fixmath}

\newtheorem{Definition}{Definition}
\newtheorem{Example}{Example}
\newtheorem{Lemma}{Lemma}

\definecolor{newcolor}{rgb}{.8,.349,.1}

\journal{Knowledge-based Systems}

\begin{document}
	\begin{frontmatter}
		\title{Inductive Conformal Recommender System}
		
		\author[nit]{Venkateswara Rao Kagita}
		\ead{585venkat@gmail.com}
		\author[mec]{Arun K Pujari}
		\ead{akpujari@curaj.ac.in}
		\author[hcu]{Vineet Padmanabhan}
		\ead{vineetcs@uohyd.ernet.in}
		\author[DU]{Vikas Kumar~\corref{cor1}}
		\ead{vikas007bca@gmail.com}
		\cortext[cor1]{Corresponding author}
		%
		% \address[venk]{One Convergence, Hyderabad}
		\address[nit]{National Institute of Technology, Warangal, India}
		\address[mec]{Mahindra University,  Hyderbad, India~~~~~~~~~~~~~~~~~} 
		\address[hcu]{University of Hyderabad, Hyderbad, India~~~~~~~~~~~~~}
		\address[DU]{University of Delhi, Delhi, India~~~~~~~~~~~~~~~~~~~~~~~~~}
		%\address[curaj]{Central University of Rajasthan, Rajasthan, India~~~~}
		
		%\cortext[cor1]{Please address correspondence to Author2 or Author3}
		
		\begin{abstract}
			Traditional recommendation algorithms develop techniques that can help people to choose desirable items. However, in many real-world applications, along with a set of recommendations, it is also essential to quantify each recommendation's (un)certainty. The conformal recommender system uses the experience of a user to output a set of recommendations, each associated with a precise confidence value. Given a significance level $\varepsilon$, it provides a bound $\varepsilon$ on the probability of making a wrong recommendation. The conformal framework uses a key concept called \emph{nonconformity measure} that measures the strangeness of an item concerning other items. One of the significant design challenges of any conformal recommendation framework is integrating nonconformity measures with the recommendation algorithm. This paper introduces an inductive variant of a conformal recommender system. We propose and analyze different nonconformity measures in the inductive setting. We also provide theoretical proofs on the error-bound and the time complexity. Extensive empirical analysis on ten benchmark datasets demonstrates that the inductive variant substantially improves the performance in computation time while preserving the accuracy. 
			
			\vspace{0.2cm}
			\noindent Keywords:~ Recommender System, Inductive Conformal Prediction, Conformal Recommender System
		\end{abstract}

	\end{frontmatter}
	
	\section{Introduction}	
	Recommending quality services to improve customer satisfaction is of prime concern for the overall success of any online community. In this context, an automatic recommendation has become even more indispensable. Recommender systems are software tools that use past behaviour (usage information) of individuals to provide personalized recommendations for a large variety of available products such as movies, books, music, etc. There have been numerous research proposals on recommendation problem focusing on improving recommendation accuracy~\cite{cheng2016wide,karatzoglou2017deep, kumar2017collaborative}. With the upcoming importance on accountability and explainability of AI techniques, deployment of a plain recommendation whatsoever accurate it may be on a testing platform will not be satisfactory without value additions such as explanations, confidence, or sensitivity. Among the desirable features of the future of recommender systems, providing a confidence measure (or, equivalently, a probable error bound) on recommendation is essential. Most of the existing recommender systems do not offer any such measure to indicate the level of confidence till very recently when the present authors propose \textit{Conformal Recommender Systems (CRS)}~\cite{kagita2017conformal}. Though some of the earlier systems endeavour to provide confidence ~\cite{mazurowski2013estimating, hernando2013incorporating,mclaughlin2004collaborative}, the confidence values so provided are not related to or bound to the error values. On the other hand, \textit{Conformal Recommender Systems (CRS)} satisfies a \emph{validity} property  that ensures  that the error value does not exceed a predetermined significance level $\varepsilon$. In other words, the correctness-confidence of a recommendation is 1-$\varepsilon$. It is observed that though CRS is an advancement in research in the area of Recommender Systems, the underlying process is computationally intensive and expensive. Having established the point that a \emph{valid} quantitative measure of confidence can be computed using the principles of conformal prediction, the need arises to provide a computationally efficient method of accomplishing this task. The objective of the present work is to investigate  efficient alternative techniques retaining the strength of CRS. This paper proposes an \textit{inductive} variant of a conformal recommender system that is computationally efficient and retains the validity property of CRS.

	For a set of training examples $S = \{z_{1}, \ldots, z_{n}\}$, where  $z_{i}$ is a pair $(x_{i}, y_{i})$ with $x_{i} \in \mathbb{R}^{d}$ is a vector of $i^{th}$ example and $y_{i}$ is the corresponding class label,  any common predictor predicts a class label $y_{n+i}$ for  unclassified objects $x_{n+i}, i\ge 1$. In contrast, conformal predictors give a set of class labels as prediction regions and corresponding probability-bounds of error. A (1-$\varepsilon$) confidence prediction region is defined by the probability that the correct label is not in the prediction region does not exceed $\varepsilon$. To predict the class label of an unclassified object, say $x_{n+1}$ to one of the class labels, say $y_c$, based on the available information in terms of $S$,  conformal predictors define suitable numeric measure to compute a nonconformity measure for each pair of  training example and class-label.
	%conformal predictors define suitably and compute a numeric measure, \VK{i.e.,} nonconformity measure for each pair of class-label and training example.
	%
	Intuitively, it is a measure of how well an unknown data $x_{n+1}$ conforms to any training example $x_{i}$ when $x_{n+1}$ is assigned class label $y_c$. This is done by measuring the change in predicting behaviour of $S$ when $z_i$ is replaced by $z_{n+1}$. The predicting behaviour is observed by applying any of the conventional predictors. The nonconformity measures for all such pairs are analysed to compute  $p$-values and then to determine (1-$\varepsilon$) region subsequently. 
	
	Two important observations can be made from the foregoing discussion. First, the process is hinged on the definition of suitable nonconformity measure. We observe that depending on the context, it is sometimes easier to use a conformity measure instead of a nonconformity measure, though both processes are equivalent intuitively. For the sake of notational convenience, we use the term nonconformity measure to refer to both situations. Second, the measure is required to be computed for all pairs of $x_{i}$ and $y_c$ in order to determine the p-values. 
	In order to show the relevance and applicability of the principle of conformal prediction, 
	a nonconformity measure is introduced by Kagita et al.~\cite{kagita2017conformal} in the context of recommender systems using precedence information.  Based on the rating data of a set of users on a set of items, a nonconformity measure is calculated for all possible recommendations by examining how well the tentative recommendation conforms to all other known recommendations and earlier ratings for any user. The underlying algorithm uses precedence mining as proposed in ~\cite{ParameswaranKBG10}.  A different nonconformity measure is defined in~\cite{himabindu2018conformal}, wherein the matrix factorization is used as the  underlying algorithm. 
	
	The main contributions of the present work are as follows. First, we analyze different possibilities of defining nonconformity measures in the context of inductive CRS by using the \emph{precedence relations} among items. As stated earlier, defining a conformity measure is observed to be more relevant than a nonconformity measure in some situations. We adopt different probability measures using pairwise precedence statistics characterized by Parameswaran et al.~\cite{ParameswaranKBG10} for defining suitable (non)conformity measures. Second, we introduce the concept of \emph{inductive conformal recommendation}, which is a computationally efficient alternative to the CRS framework. %Further, we establish the crucial properties of the conformal framework, validity and efficiency, theoretically and empirically. 
	Further, we theoretically and empirically establish the crucial properties of the conformal framework, i.e., validity and efficiency. To verify its efficacy, we conducted
	extensive experiments %We offer extensive experimental analysis 
	on seven bench-mark datasets using various standard evaluation metrics. We show that the proposed inductive conformal recommender system improves the computational time while retaining a similar level of accuracy. 
	
	The rest of the paper is structured as follows. In Section~\ref{sec:relatedWork}, we briefly discuss the related work. Section~\ref{sec:FC} describes the key concepts required to build the proposed system. Section~\ref{sec:conformalPrediction} presents the background on conformal prediction framework.  In Section~\ref{sec:PM}, we discuss the underlying precedence mining based recommender systems. We discuss the existing conformal recommender system in Section~\ref{sec:crsIS}. We introduce the proposed inductive conformal recommender system in Section~\ref{sec:ICRS}. We report experimental results in Section~\ref{sec:Exp}. Finally, Section~\ref{sec:conclusion} concludes and indicates several issues for future work.

	\section{Related Work: Confidence Measure in Recommender System}
	\label{sec:relatedWork}
	Recommender systems are generally employed to provide tailor-made suggestions that can assist the user in decision making~\cite{ resnick1997recommender, lu2015recommender}. These systems exploit the user's consumption experience collected via implicit or explicit feedback data to infer their preferences~\cite{oard1998implicit, kumar2017proximal, margaris2019improving}. However, most of these systems are less transparent because of the unavailability of confidence with which an item is recommended~\cite{himabindu2018conformal, ayyaz2018hcf}. Despite the enormous application of recommender systems, a limited number of methods are available that associate confidence value with the recommendations. In this section, a brief review of the earlier work concerning confidence measures in the recommender system is presented. Readers' familiarity with recommender system is assumed here. 
	
	To measure the effect of confidence and uncertainty measures, McNee et al.~\cite{mcnee2003confidence} involved an elementary confidence computation in existing collaborative filtering algorithms and have shown that a confidence display increases user satisfaction. In~\cite{mclaughlin2004collaborative}, the authors have considered the previously collected user's rating as noisy evidence of the user's actual rating and proposed a Belief Distribution algorithm that explicitly outputs the uncertainty in each predicted rating along with the predicted rating value. Adomavicius et al.~\cite{adomavicius2007towards} proposed a rating variance-based confidence measure to refine the prediction generated by any traditional recommendation algorithm. Symeonidis et al.~\cite{symeonidis2008providing} constructed a feature profile of each user, and then the prediction is justified by considering the correlation between users and features. Shani et al.~\cite{shani2011evaluating} suggested measuring the significance level of recommendation by running a significance test between the results of different recommender algorithms. OrdRec~\cite{koren2011ordrec} provides a richer expressive power by producing a full probability distribution of the expected item ratings. Mazurowski~\cite{mazurowski2013estimating} compared the concept of confidence in collaborative filtering with similar concepts in other fields within machine learning. Bayesian confidence intervals-based evaluation method has been proposed to measure recommendation algorithms' performance. The author also proposed three different resampling-based methods to estimate the confidence of individual predictions~\cite{mazurowski2013estimating}. In~\cite{hernando2013incorporating}, for a target user, the confidence in prediction for an item is defined based on \textit{k}-nearest neighbors of the user. 	In~\cite{ayyaz2018hcf}, a content-based fuzzy recommendation model is proposed that utilize \textit{similarity} and \textit{dissimilarity} score between user and item for the rating prediction task. For every unknown (user, item)-pair, the prediction confidence is computed based on the difference between the actual ratings given by that user and their corresponding predictions by the fuzzy model. A recommendation model is proposed in~\cite{gohari2018new} to integrate the trust and certainty information for confidence modeling. Mesas et al.~\cite{mesas2020exploiting} explored the prediction confidence from the perspective of the system. The idea is to embed awareness into the recommendation models that help in deciding the more reliable suggestions rather than all potential recommendations. A Course Recommender system is proposed in~\cite{morsomme2019conformal}, where a course-specific regression model is trained over the course contents and students’ academic interests for the grade predictions. To complement the model predicted grades, the authors have employed an Inductive Confidence Machine (ICM)~\cite{papadopoulos2002inductive} to construct prediction intervals attune with each student. In~\cite{himabindu2018conformal}, two variants of conformal framework, namely transductive and inductive, are proposed in the matrix factorization (MF) setting that associate a confidence score to each predicted rating.  The method proposed in~\cite{himabindu2018conformal} can be seen as a two-stage procedure. At first, a MF-based model is applied over the partially filled rating matrix to get the rating prediction for each (user, item)-pair. These predictions are then used to calculate the confidence score for individual predicted ratings.   A confidence-aware MF model is proposed in~\cite{wang2018confidence}, which can be seen as a comprehensive framework that optimizes the accuracy of rating prediction and estimates the confidence over predicted rating simultaneously. Costa et al. \cite{ da2019boosting} proposed an ensemble-based co-training approach for the rating prediction problem. In the co-training phase, two or more recommender algorithms are trained to predict the rating for all unobserved user-item pairs. The training set for the next iteration of the co-training is then augmented with the $M$ most confident predictions. The confidence is calculated based on the deviation from the baseline estimate and the rating predicted by the recommendation algorithm.
	\emph{However, none of these works provide confidence to the recommendation set. They focus on providing confidence to the individual rating prediction, and it is non-trivial and cumbersome to obtain the confidence of recommendation from confidence regions of point predictions. }
	
	In this work, we focus on providing confidence to the recommendation, not for rating prediction. The only work that focuses on providing confidence to the recommendation is our previous work on conformal recommender system~\cite{kagita2017conformal}, wherein a conformal framework is introduced for the recommender systems, and  a new nonconformity measure is proposed for the conformal recommender system. It is also shown that the proposed nonconformity satisfies the desirable properties of conformal prediction, such as \textit{exchangeability}, \textit{validity}, and \textit{efficiency}. Nonetheless, the framework proposed in~\cite{kagita2017conformal} suffers from similar shortcomings of traditional conformal predictions and requires high computation times. %\VK{that is the computational time is very high.} 
	%\VK{ namely, in terms of computational time}. 
	We briefly describe the approach in the following section.

	\section{Foundational Concepts}
	\label{sec:FC}
	In this section, we first introduce the basic concepts related to \emph{conformal prediction}, the main framework we use to build  our 
	proposed confidence-based recommender system. We then give a brief description of \emph{precedence mining}, a collaborative filtering model, on which we apply our conformal prediction framework for producing confidence-based recommendations.
	
	\subsection{Conformal Prediction}
	\label{sec:conformalPrediction}
	In this section,  a brief account of the principle of conformal prediction is reported in order to provide the relevant background. 
	We start with a training example $z_{i}$ as a pair $(x_{i}, y_{i})$ with $x_{i} \in \mathbb{R}^{d}$ is a feature vector of $i^{th}$ example and $y_{i}$ is the corresponding class label. %Let $S = \{z_{1}, \ldots, z_{n}\}$ be the training set.  
	Given the training set $S = \{z_{1}, \ldots, z_{n}\}$, a prediction or classification task is to predict a class label $y_{n+i}$ for  unclassified objects $x_{n+i}, i\ge 1$. A conformal predictor provides a subset of class labels for each unclassified object $x_{n+i}$ and the error that the correct label is not in this set does not exceed $\varepsilon$. Let us consider one unclassified object $x_{n+1}$ and the task is to examine whether a class label $y_c$ is a member of $(1-\varepsilon)-$prediction region. %\VK{ for $x_{n+1}$}. 
	Let $z^{c}_{n+1} = (x_{n+1}, y_c)$, where  $y_c$ is tentatively assigned to $x_{n+1}$. The nonconformity measure for an example $z_i \in \{S\cup z^{c}_{n+1}\}$ is a measure of how well  $z_i$ conforms to $\{S\cup z^{c}_{n+1}\} \setminus z_i, \forall i\in [1,n+1]$.  From another point of view, it can be seen as a measure of how well $z^{c}_{n+1}$ conforms to $z_{i}\in S$. This is done by measuring the change in predicting behaviour of $S$ when $z_i$ is replaced by $z^{c}_{n+1}$. The $p$-value is the proportion of $z_{i}\in S$ having nonconformity score worse than that of $z^{c}_{n+1}$ for all possible values of $y_c$ (all class labels). The set of labels whose $p$-value higher than $\varepsilon$ forms $(1-\varepsilon)-$prediction region. Intuitively, the predicting behaviour is observed by applying any of the conventional predictors which uses S as the training set. The conformal prediction algorithm makes $(n+1)\times n_c\times C$ calls to the underlying prediction algorithm, where $n_c$ is the number of candidate items, and $C$ is the number of class labels. % and \VK{ this becomes the dominating factor in worst-case complexity of the process irrespective of which prediction algorithm is used}. 
	%It may be noted that $S$ is used for training of the underlying algorithm too. 
	The conformal prediction framework has been well-studied from different perspectives in recent years~\cite{papadopoulos2008inductive, zenebe2005personalized,shafer2008tutorial,vovk2005algorithmic}.
	
	On the other hand, the \emph{inductive conformal framework} avoids the computational overhead~\cite{papadopoulos2008inductive} of initial proposal of conformal prediction. 	In an inductive setting the training set $S = \{z_{1}, \ldots, z_{n}\}$ is divided into two sets, namely \emph{proper} training set $S^t = \{z_1, z_2, \ldots, z_m\}$ and \emph{calibration} set $S^c = \{z_{m+1}, \ldots, z_{m+l}\}$, $n = m+l$. The former is used to learn the prediction model and the latter is used for computation of $p$-values. The system uses an underlying conventional prediction algorithm to learn a model using proper training set $S^t$. The same model is then used to determine (non)conformity measures and $p$-value for every example in $S^c$ and $z_{n+1}$ with respect to $S^t$. As a result, the framework learns the underlying model only once, leading to a significant reduction in computation time and effort. 
	
	\begin{Example}
		Consider a  problem of classifying samples as cancerous ($+ve$) or noncancerous ($-ve$) based on the tumor size and other pathological features. Let $x_i$ be the feature vector describing an $i^{th}$ instance, and $y_i\in\{+ve, -ve\}$ is the corresponding label. Let  $S=\{(x_1, y_1), (x_2, y_2), \ldots, (x_{10}, y_{10})\}$ be the training set containing observations of ten different individuals. Given the new instance, say $x_{11}$, the task is to classify it as either $+ve$ or $-ve$.  %Let the underlying model be SVM with nonconformity value as the deviation between actual and predicted for wrongly classified instances and is zero for correctly classified cases. 
		Assume, Support Vector Machine (SVM) is the underlying classifier that \VK{also} assigns a nonconformity value as the deviation between the actual and the predicted class label.
		The conformal prediction framework labels the new instance with all possible classes and sees which conforms more to the existing ones. At first, it considers $z_{11}^{+ve} = (x_{11}, +ve)$ and adds it to the training set. After that the the SVM classifier is trained for each $i^{th}$ instance with the training set $\{S\cup z_{11}^{+ve}\} \setminus z_i, \forall i\in [1,11]$ and measures the removed example's nonconformity ($\alpha_i$).
		%It then removes one instance $i\in[1,11]$ at a time from the dataset, trains SVM with the rest of the dataset, and measures the removed example's nonconformity ($\alpha_i$). 
		Finally, it calculates the p-value concerning the label $C$ (let's say $P^{+ve}$), $P^{+ve} = \frac{\lvert  i|1\le i\le11, \alpha_i\ge \alpha_{11}\rvert}{11}$. 
		The conformal predictor repeats the same procedure concerning another label $NC$, i.e., for $(x_{11}, -ve)$ and determines the corresponding p-value ($P^{-ve}$). The prediction region includes the labels with a p-value greater than the significance level $\varepsilon$. We can also observe that for the given example, the conformal predictor requires training of $22$ $\big(11\times 2 = (n+1)\times C\big)$ SVMs for one candidate item and hence, for $n_c$ candidate items it would be $(n+1)\times C \times n_c$ SVMs.
		
		On the other hand, the Inductive Conformal Predictor divides the dataset into two sets, namely proper training set $S^t$ and calibration set $S^c$. %Assume $S^t = \{(x_1, y_1), (x_2, y_2), \ldots (x_7, y_7)\}$, and $S^c = \{(x_8, y_8), \ldots (x_{10}, y_{10})\}$. 
		It then trains the underlying model with $S^t$ and uses the same to evaluate the nonconformity of $S^c$ and new instance $x_{11}$. Hence, only one SVM classifier is learnt using the set $S^t$ and called for $\lvert S^c\rvert + C$ times. For $n_c$ candidates it is $n_c\times (\lvert S^c\rvert + C)$, which is a drastic improvement over conformal predictor in terms of computation time. 
	\end{Example}
	
	\subsection{Precedence Mining based Recommender Systems~\cite{ParameswaranKBG10}}
	\label{sec:PM}
	The precedence mining model~\cite{ParameswaranKBG10} is a Collaborative Filtering (CF) based model that maintains precedence statistics, i.e., the temporal count of all the pairs of items.
	The precedence mining model estimates the probability of future consumption based on past behaviour. For example, a person who has seen \textit{Godfather I} is more likely to watch \textit{Godfather II} in the future. %Though precedence mining is a collaborative filtering model it differs from the standard collaborative filtering approach.
	%The precedence mining based approaches  differs from traditional approaches collaborative filtering modls
	In most of the traditional CF techniques, the aim is to find users having similar profiles as the active user $u$, and then restrict its search to items consumed by this subset of users and not consumed by $u$.  Thus, certain consumption patterns of items exhibited by the whole set of users are not captured as the search is restricted.  The precedence mining model overcomes these shortcomings and  attempts to capture pairwise precedence relations frequently occurring among all users. 
	%The aim here is to 
	It calculates a recommendation score for each item based on the precedence statistics, and then the set of items having scores greater than the threshold are recommended. 
	\begin{Example}
		Figure~\ref{fig:ch2_ex1} shows the difference between traditional collaborative
		filtering and precedence mining. The leftmost table in the figure is a toy example
		%  that we took to demonstrate the difference between collaborative and precedence mining approaches. 
		in which we provide the profiles of different users. Let $u_a$ denote the active user $u_a$. Each row in the table can be interpreted as 
		a sequence of movies that the user has watched. For instance, $u_a$ has watched 
		movies $m_1,~ m_2,~ and~ m_3$ in the given order. The figure in the middle demonstrates the working of collaborative filtering. Here, we assume that the set of users who have at least two movies in common with the active user are in its neighbours. The  most popular movie among the neighbours are then recommended to the active user. % We assume here  wherein the idea is to consider those users as neighbours who have at least two movies in common between them and thereafter recommend the most popular item among the neighbours. 
		%We take the users who share two common movies as neighbours and recommend the item popular among neighbours. 
		A careful observation of Figure~\ref{fig:ch2_ex1} reveals that the movie $m_5$ is popular among the neighbours $u_1$ and $u_2$
		%If we see the figure movie $i_5$ is popular among the neighbours $u_1$ and $u_2$. 
		and therefore collaborative filtering recommends movie $m_5$ to the user $u_a$. It is to be noted that the search of items in collaborative filtering is limited to 
		the neighbours space. In contrast, precedence mining looks for patterns in which one item follows the other in the whole user space. The rightmost image in the Figure~\ref{fig:ch2_ex1} demonstrates the idea of precedence mining. We highlight the patterns that occur at least thrice using different colors, for instance, $m_1$ and $m_7$. %and different patterns are  highlighted with different color shades. 
\begin{figure}[!ht]
	% \centering
	\begin{minipage}[b]{0.31\linewidth}
		\centering
		\includegraphics[scale = 0.37]{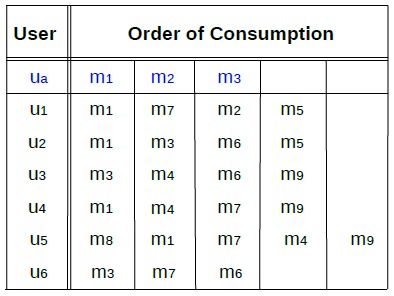}
	\end{minipage}
	\hfill
	\begin{minipage}[b]{0.31\linewidth}
		\centering
		\includegraphics[scale = 0.37]{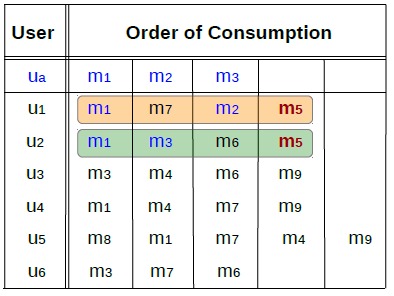}
	\end{minipage}
	\hfill
	\begin{minipage}[b]{0.31\linewidth}
		\centering
		\includegraphics[scale = 0.37]{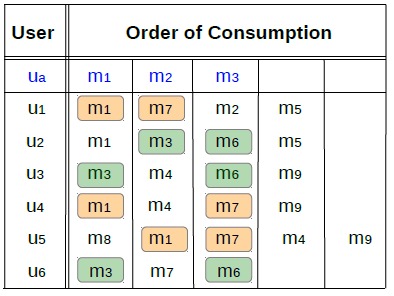}
	\end{minipage}
	\caption{Comparison of collaborative filtering (middle) and precedence mining (right) approaches for a toy example (left).  }
	\label{fig:ch2_ex1}
\end{figure}
\end{Example}

Recommender systems based on
precedence relations is concerned with
mining precedence relations among items  consumed by users and thereafter recommends new items having high \emph{relevance score} computed using precedence statistics. The nicety and novelty of this approach is the use of pairwise precedence  relations  between items. We describe the score computation formally as follows.

Let $O= \{o_{1}, o_{2}, \ldots , o_{mobject}\}$ be the set of items and $U = \{u_{1}, u_{2}, \ldots ,u_{muser}\}$ be the set of users. $profile(u_{j})$ is a sequence of items  known
to have been consumed by user $u_{j}$.  Let $O_{j}$ be the set of items consumed by $u_j$.
A recommender system is concerned with recommending items to a user based on profiles of different users. A recommender system aims at selecting items for recommendation such that these items are absent in
$profile(u)$ and are expectedly preferred to
other items by the user for whom it is recommended.
Let $Support(o_i)$ be the number of users that have consumed item $o_{i}$ and $PrecedenceCount(o_i, o_h)$ be the number of users having consumed item $o_{i}$ preceding
$o_{h}$. The precedence probability for item $o_{i}$ preceding $o_{h}$ is denoted as $PP(o_{i}|o_h)$. We define  $PP(o_{i} | o_{h})$, and $Score(o_{i},
u_{j})$ as follows.
\begin{equation}
PP(o_i|o_h) = \frac{PrecedenceCount(o_i, o_h)}{Support(o_h)},
\end{equation}
%   and
\begin{equation}
Score(o_i, u_j) =  \frac{Support(o_i)}{muser} \times \underset{o_l \in
	O_j}{\operatorname{\prod}}PP(o_l|o_i).
\label{eq4}
\end{equation}
The objects with high score  are recommended.
%It is shown in~\cite{KagitaPP15} that $Score(o_{i},
%u_{j})$ is invariant of the permutation of items in $profile(u_j)$.
If the score for a  given unutilized item is low, it is highly unlikely to be of interest to the user.
We now consider an example which illustrates the working of the preceding precedence mining based recommender system. 
\begin{Example}
Consider the following \emph{PrecedenceCount} and \emph{Support} statistics calculated based on the preferences given by thirty users $U =
\{u_1, u_2, \ldots, u_{30}\}$ over ten items $O = \{o_1, o_2, \ldots, o_{10}\}$.

\[
PrecedenceCount = \left[
\begin{array}{cccccccccc}
	0 & 9 & 8& 11&  7 & 8 & 6 & 7 & 7 & 3 \\
	8 & 0& 10 &11 & 9 & 7 & 7 & 6 & 8 & 4 \\
	8 & 8 & 0 & 5&  7 & 6 & 5 & 6 & 4 & 3 \\
	5 & 11 & 12 & 0 & 6 & 8 & 6 & 6 & 3 & 2 \\
	7 & 9 & 7 & 13 & 0 & 8 & 7 & 6 & 9 & 4 \\
	5 & 9 &6 & 8 & 6 & 0 & 5 & 6  & 4 & 1 \\
	4 & 6 & 5 & 7 & 5 & 5 & 0 & 4 & 4 & 1 \\
	4 & 8 & 6 & 10 & 8 & 6 & 5 & 0 & 7 & 2 \\
	7 & 11 & 10 & 16 & 7 & 7 & 6 & 5 & 0 & 3 \\
	2 & 1 & 0 & 2 & 1 & 2 & 1 & 3 & 1 & 0 \end{array} \right] \]
\[
Support =  \left[
\begin{array}{cccccccccc}
	20 &  25 &  21 &  25 &  22 &  18 &  15 &  18 &  20 &  6
\end{array}\right].
\]

\noindent Let $u_1$ be the target user and $O_1 = \{o_1,~ o_3,~ o_5,~ o_7,~ o_9\}$ be set of items consumed by $u_1$. For $u_1$, the candidate items for recommendation are $O\setminus O_1 = \{o_2, o_4, o_6, o_8, o_{10}\}$.
The score of an item $o_2$ which not consumed by user $u_1$ is then calculated as
\begin{align*}
	Score(o_2, u_1) &= \frac{Support(o_2)}{30}\times PP(o_1\mid o_2)\times PP(o_3\mid o_2)\times PP(o_5\mid o_2)\times PP(o_7\mid o_2) \times PP(o_9\mid o_2)\\
	&= \frac{25}{30} \times \frac{9}{25}\times\frac{8}{25}\times \frac{9}{25}\times \frac{6}{25}\times  \frac{11}{25} = 0.0036.
\end{align*}

\noindent Similarly,  $Score(o_4, u_1) = 0.0068$,  $Score(o_6, u_1) = 0.0043$, $Score(o_8, u_1) = 0.0016$, and
$Score(o_{10}, u_1) = 0.0028$. Hence, it ranks the items in the order of $o_4,~ o_6,~ o_2,~ o_{10}$, and $o_{8}$. 
\end{Example}

The problem with this approach is even if one of the precedence probabilities (PPs) is zero, the whole score becomes zero. To avoid this problem, Parameswaran et al.~\cite{ParameswaranKBG10}  proposed to consider only top-I precedence probabilities in the product term, where $I$ is a hyper-parameter to tune. In our experiments, we have fine tune the value $I$  to be $1$.

\section{Conformal Recommender System}
\label{sec:crsIS}
The principle of conformal prediction is applied to recommender system in~\cite{kagita2017conformal}. Here, we briefly report the proposal of CRS.  The readers are requested to refer~\cite{kagita2017conformal} for details. Let $O$ be the set of items, $n_i=\lvert O \rvert$ be the total number of items, $n_u$ be the number of users and $O_{j} = \{o_1, o_2, \ldots, o_{n}\}$ be the set of items consumed by a user $u_j$. Given $O,~ u_j,~ O_j,$ and the significance level $\varepsilon$, the problem is to recommend a set of items $\Gamma^\varepsilon$ with $(1-\varepsilon)$ confidence. 
%
%Let $Sup(o_i)$ be the number of users who have consumed an item $o_{i}$ and $PC(o_i, o_h)$ be the number of users having consumed item $o_{i}$ preceding $o_{h}$. The precedence probability for item $o_{i}$ preceding $o_{h}$ is denoted as $PP(o_{i}|o_h)$ and is defined as $PP(o_i|o_h) = \frac{PC(o_i, o_h)}{Sup(o_h)}$.
For a given user $u_j$, %$O_{j}$ is the set of all utilized objects by the user $u_j$ and  
$O_{j}$ is split into two sets based on the precedence of usage of the items. The first set $O_j^{train} = \{o_1, o_2, \ldots, o_n\}$ is used as the training set.  The set $O_j^{candidates} = \{c_1, c_2, \ldots, c_k\}$, of candidate items that are consumed by $u_j$ after use of items in  $O_j^{train}$ and are not part of the training set. The conformal recommendation process is to determine the confidence measure of recommending a new object $o_{n+1}$ for user $u_j$. The first step of CRS is to see how well the object $o_{n+1}$ and the training set $O_j^{train}$ conform to each other. %We append $o_{n+1}$ to $O_j^{train}$ and denote the appended set as $O_j^{train+}$.
Let $O_j^{train+} = \{O_j^{train} \bigcup o_{n+1}\}$ be the appended set.
Nonconformity measure is computed for each $o_{h} \in O_j^{train+}$ by ignoring $o_{h}$ in $O_j^{train+}$ and examining the recommendability of $c_i$ when the profile is $O^{h}_{j}$,  where
\begin{equation}
O^{h}_{j} = O_j^{train+} \setminus \{o_{h}\} = \{ o_{1}, o_{2}, \ldots, o_{h-1}, o_{h+1}, \ldots, o_{n+1}\}.
\end{equation}

With precedence mining~\cite{KagitaPReMI,KagitaAAI, KAGITA201515} as the \textit{underlying algorithm}, the measure of recommendability of $c_i$  is a numerical value, $Score(c_i, u_j)$
and higher value of $Score$ implies greater chance of being recommended. The $Score$ is calculated with reference to each of $h$ and for each tentative profile $O^{h}_{j}$, $Score^{h}$ is defined as
\[
\alpha_h = Score^{h}(c_i, u_j) =  \frac{Sup(c_i)}{m_{user}} \times \underset{o_l \in
O^{h}_{j}}{\operatorname{\prod}}PP(o_l|c_i).
\]

%The difference between $Score^{h}(c_i, u_j)$ and $Score^{n+1}(c_i, u_j)$ indicates the nonconformity of $o_{n+1}$ with respect to $o_{h}$. The following definition is from~\cite{kagita2017conformal}

\begin{Definition}
(CRS nonconformity measure~\cite{kagita2017conformal}). Given a subset $O_j^{train}$ of user $u_j$ profile;  a set of objects $O_j^{candidates} = \{c_1, c_2, \ldots, c_k\}$, that are consumed by $u_j$ after use of items in  $O_j^{train}$ and are not part of the training set; and a new object  $o_{n+1} \in  O_j$, the nonconformity measure $ \mathcal{A}(o_1, o_2, \dots, o_{n+1})$ w.r.t. $c_i \in O_j^{candidates}$   is $(\alpha_1,\alpha_2,\dots,\alpha_{n+1})$, where $\alpha_h = Score^{h}(c_i, u_j)$.
\end{Definition}
The computed nonconformity scores $\alpha_h, h\in[1, n+1]$  are used to compute the $p$-value as the proportion of examples with $\alpha_h\ge \alpha_{n+1}, h\in[1, n+1]$. A $p$-value is computed for each $c_i \in O_j^{candidates}$ and then we employ two different aggregation techniques to select the final $p$-value from several $p$-values. If the selected $p$-value is greater than $\varepsilon$, then $o_{n+1}$ is included in the $(1-\varepsilon)$ confidence recommendation region. The procedure is repeated for every new item $o_{n+i}, i\ge1$ to get  $(1-\varepsilon)$ confidence recommendation set.

\begin{Example}
We consider the precedence statistics given in Example 3 for this example also.   
Let $O_1^{train} =
\{o_1,~ o_3,~ o_5\}\subset O_1$ and
$O_1^{candidates} = \{o_7, o_9\}$. Let $o_2$ be the candidate item for recommendation. We append $o_2$
with $O_1^{train}$ as $O_1^{train+} = \{o_1,~ o_3,~ o_5,~ o_2\}$.
Nonconformity of an item $o_h\in O_1^{train+}$ is measured by the recommendability of a candidate item $c\in O_1^{candidates}$ using the profile $O_1^{train+}\setminus \{o_h\}$.
For example, nonconformity of an item $o_1$ concerning the recommendability of $o_7$ is computed as
\begin{align*}
	\alpha_1 = Score^1(o_7, u_1) & = \frac{Support(o_7)}{30} \times PP(o_3\mid o_7)
	\times PP(o_5\mid o_7) \times PP(o_2\mid o_7) \\
	& =\frac{15}{30}\times\frac{5}{15}\times\frac{7}{15}\times\frac{7}{15}=0.036.
\end{align*}
Nonconformity score of $o_3$ is
\[
\alpha_3 = Score^3(o_7, u_1) = \frac{Support(o_7)}{30} \times PP(o_1\mid o_7)
\times PP(o_5\mid o_7) \times PP(o_2\mid o_7) = =0.043.
\]
Similarly, Nonconformity score of $o_5$ is  $\alpha_5 = Score^5(o_7, u_1) =
0.031$ and nonconformity score of $o_2$ is
$\alpha_2 = Score^2(o_7, u_1) = 0.031$. The p-value of $o_2$ concerning the recommendability of $o_7$ is computed as follows. 
\[
p(o_2, o_7) = \frac{\Big\lvert\big\{ o_h \big\vert  o_h\in O_1^{train+}, Score^h(o_7, u_1)\geq Score^2(o_7, u_1)\big\}\Big\rvert}{\lvert O_1^{train+}\rvert}
= \frac{4}{4} = 1,
\]
Similarly, we compute the p-value of $o_2$ concerning the recommendability of $o_9$ that is $p(o_2, o_9) = 0.75$. In order to get the final $p$-value from $p(o_2, o_7)$ and $p(o_2, o_9)$, \texttt{CRS-max}~\cite{kagita2017conformal} employs a \emph{maximum} strategy and \texttt{CRS-med}~\cite{kagita2017conformal} employs a \emph{median} strategy. Therefore, the final $p$-value according to \texttt{CRS-max} and \texttt{CRS-med} are 1 and 0.875 respectively. Similarly, we compute the $p$-value for all the candidate items for recommendation and recommend the items whose $p$-value is greater than $\varepsilon$ with the confidence of $(1-\varepsilon)$.

%$p(o_4, o_7) = \frac{3}{4} = 0.75, p(o_6, v) = \frac{2}{4} = 0.5, p(o_8, v) =
%\frac{2}{4} = 0.5, p(o_9, v) = \frac{3}{4} = 0.75$,
%and $p(o_{10}, v) = \frac{1}{4} = 0.25$.
\end{Example}

\section{Inductive Conformal Recommender System}
\label{sec:ICRS}
This section presents the proposed \emph{inductive conformal recommender system (ICRS)} to gauge the confidence of recommendations. The proposed conformal approach determines a recommendation set $\Gamma^\varepsilon$ with  $(1-\varepsilon)$ confidence for a given significance level $\varepsilon$. 
A pivotal component of the conformal framework is the nonconformity measure quantifying the reliability in prediction. We use precedence relations among the items to determine the nonconformity score. Precedence relations capture the temporal patterns in user transactions. Besides, precedence relations based recommender systems do not require rating information, which is indeed challenging to obtain in a real-time scenario. Furthermore, these systems are ranking systems and thus allow us to define confidence for recommendation instead of a rating prediction. These are the various reasons for choosing precedence relations to represent nonconformity measures. 

%\textcolor{red}{ PLEASE CHECK THE FLOW}
%\subsection{}
The brief idea of the proposed approach is as follows. We split $O_{j}$ into \emph{proper training set} $O_j^t = \{o_1, o_2, \ldots, o_{m}\}$ and  \emph{calibration set} $O_j^c = \{o_{m+1}, o_{m+2}, \ldots, o_{m+l}\}$, wherein  $O_j^c$ is the set of items known to be consumed after $O_j^t$ and $n=m+l$. The idea is to compute the (non)conformity measure for every item in the calibration set along with a new item $o_{n+1}$ and determine $o_{n+1}$'s \emph{p-value}: the proportion of items having (non)conformity score better than or equal to that of a new item. Subsequently, we include item $o_{n+1}$ in the $\Gamma^\varepsilon$ recommendation region if the $p$-value of $o_{n+1}$ exceeds $\varepsilon$.  

%We explore different possibilities of measuring (non)conformity scores using various precedence relations among items and show that these measures are invariant of permutation. In our recent work~\cite{kagita2017conformal}, we successfully adapted precedence relations to define a  nonconformity measure in a transductive setting.  

The following subsections elaborate on the notions of (non)conformity measures and the p-value and describe the complete procedure. % The rest of the section is organized as follows. Subsection~\ref{problemFormulation} introduces the basic  notations and definitions used throughout the section.
Subsection~\ref{sec:NCM} defines the various \linebreak (non)conformity measures to determine the conformity or strangeness of an object concerning the training set. Subsection~\ref{sec:pvalue} defines $p$-value, which quantifies the conformity score of a new item concerning the training set of items and defines the recommendation set $\Gamma^\varepsilon$ with $(1-\varepsilon)$ confidence.   Subsection~\ref{Algo_Flow} gives the flowchart of the proposed system and describes the proposed algorithm.  In Subsection~\ref{sec:validity}, we describe the two important measures of any conformal prediction framework, \textit{validity} and \textit{efficiency}, in the recommender systems setting. Finally, we proffer theoretical time complexity analysis of the proposed approach against the existing methods in Subsection~\ref{timeCAnalysis}.

\subsection{Nonconformity Measures}
\label{sec:NCM}
Nonconformity measure is a measurable function $\mathcal{A}$ that determines a new object's relation with the proper training set in terms of a scalar value. There are several ways traditional algorithms can construct nonconformity measures; each of these measures defines a unique ICRS. It is worth mentioning that a particular (non)conformity measure only affects the ICRS model's efficiency, and the validity of the results remains unaffected. We propose different conformity/nonconformity measures in this section and analyze the efficiency. 
%Let $Sup(o_i)$ be the number of users who have consumed an item $o_{i}$ and $PC(o_i, o_h)$ be the number of users having consumed item $o_{i}$ preceding $o_{h}$. The precedence probability for item $o_{i}$ preceding $o_{h}$ is denoted as $PP(o_{i}|o_h)$ and is defined as $PP(o_i|o_h) = \frac{PC(o_i, o_h)}{Sup(o_h)}$. 
We use \emph{precedence count} $PC(o_i, o_h)$ and  \emph{precedence probability} $PP(o_i|o_h)$ that determines the precedence relation among items to define various (non)conformity measures. When we compute these quantities for each item in the user profile, we get multiple values. We use different aggregation techniques as a design choice to calculate the (non)conformity value using multiple precedence statistics. For the simplicity of notations, we refer to conformity measure as \textit{CM} and nonconformity measure as \textit{NCM} in the subsequent discussion.

%We see that some of measures are truly measure the \emph{conformity} instead \emph{nonconformity}, equivalently. Though the term has been referred as nonconformity measure throughout the text, we distinguish conformity measure (CM) from nonconformity measure (NCM) below for individuall cases. 

We adapt the score function proposed by Parameswaran et al.~\cite{ParameswaranKBG10} that estimates the relevance of an item to the user profile to establish the first conformity measure.  We define the conformity score of an item $o_h$ for a user $u_j$ profile as follows. 

\begin{center}
$
CM1(o_h) =  \frac{Sup(o_h)}{n_u} \times \underset{o_l \in
	O_j^t}{\operatorname{\prod^{(I)}}}PP(o_l|o_h), 
\label{eq:CM}    
$\end{center}
where $\prod^{(I)}$ denotes the multiplication of top-I quantities in the product term. We validate the algorithm for different I values and take $I$ as 1 in the experiments. The score is high when it conforms more to the training set. Note that every measure that we define here is with respect to a target user $u_j$.
%
%\noindent{\bf Conformity Measures 2, 3, 4 and 5}\\
Furthermore, we determine an object's conformity in terms of the precedence count of an item with the set of items consumed by the user. The precedence count ($PC(o_i, o_h)$) defined previously represents the number of times an item $o_h$ appeared after $o_i$ in user profiles. The higher the number, the more likely it is that $o_h$ appears after $o_i$. Hence, we use precedence count to determine a conformity measure. We compute the precedence count of an item $o_h$ to	every item $o_i$ in the proper training set $O_j^t$ of user $u_j$ and then aggregate them to get a numerical score. 
Using the different aggregation strategies such as \emph{minimum}, \emph{median}, \emph{mean} and \emph{maximum}, we arrive at the following conformity measures: \emph{CM2, CM3, CM4}, and \emph{CM5}, respectively. The detailed formulation of these measures is given in Annexure 1. 
We also use the precedence probability of an object with respect to the user profile to determine the conformity score of an object. 
%
%we determine an object's conformity in terms of the precedence probability of an item with the set of items consumed by the user. 
Precedence probability $PP(o_h\mid o_i)$ of an item $o_h$ with respect to an item $o_i$ indicates how likely an item $o_h$ follows an	item $o_i$. Hence, we use precedence probabilities of an item	$o_h$ with respect to individual items in the user profile to define the conformity measures. %We compute the precedence probability of an item $o_h$ to	every item $o_i$ in the proper training set $O_j^t$ of user $u_j$ and then aggregate them to get a numerical score. 
%Using the different aggregation strategies such as \emph{minimum}, \emph{median}, \emph{mean} and \emph{maximum}, we arrive at the following conformity measures: \emph{CM2, CM3, CM4}, and \emph{CM5}, respectively. The detailed formulation of these measures is given in Annexure 1. 
%
% 
%\noindent{\bf Conformity Measure 6, 7, 8 and 9}\\
%The precedence count ($PC(o_i, o_h)$) defined previously represents the number of times an item $o_h$ appeared after $o_i$ in user profiles. The higher the number, it is more likely that $o_h$ appears after $o_i$. Hence, we directly use precedence count to determine a conformity measure. 
We again use different aggregation strategies 
%such as \emph{minimum}, \emph{median}, \emph{mean} and \emph{maximum} 
to summarize the precedence probability scores of $o_h$ with respect to multiple items in the user profile. The process resulted in four different conformity scores, \emph{CM6 (minimum), CM7 (median), CM8(mean)}, and \emph{CM9(maximum)} with the corresponding aggregation operator mentioned in the brackets. The detailed formulation is given in Annexure 1. 
%
%\noindent{\bf Nonconformity Measures 10, 11, 12 and 13}\\
%
We then employ probability of $o_h$ given that $o_i$ is present in the target user profile without preceding $o_h$ to determine the conformity score of an item $o_h$ concerning the training data. We compute this score with respect to each and every item in the training data and employ different aggregation strategies resulting in four different conformity scores, \emph{CM10 (minimum), CM11 (median), CM12(mean)}, and \emph{CM13(maximum).}
Finally, we consider the probability that an item $o_i$ appears in the profile without succeeding an item $o_h$ $(\frac{Sup(o_i) - PC(o_i, o_h)}{n_u})$ as the potential nonconformity measure for an item $o_h$. Since there are multiple $o_i$'s in the user profile/training set, we use different aggregation strategies and define the nonconformity measures \emph{NCM14, NCM15, NCM16}, and \emph{NCM17} as given in Annexure 1.

\begin{Lemma}
\label{lem:1}
(Non)conformity of items $\{o_{m+1}, \ldots, o_{n+1}\}$ is invariant of
permutation, i.e., for any permutation $\pi$ of $\{m+1, \ldots, n+1\}$ i.e.,~$\mathcal{A}(o_{m+1}, o_{m+2}, \ldots, o_{n+1}) = (\alpha_{m+1}, \alpha_{m+2}, \ldots,  \alpha_{n+1}) \linebreak \Rightarrow  \mathcal{A}(o_{\pi(m+1)}, o_{\pi(m+2)}, \ldots, o_{\pi(n+1)}) = (\alpha_{\pi(m+1)}, \ldots, \alpha_{\pi(n+1)}).$
\end{Lemma}
\begin{proof}
It is easy to see that the nonconformity scores are invariant of permutation $\pi$ of $\{o_{m+1}, \ldots, o_{n+1}\}$. 
All the proposed conformity/nonconformity measures are independent of the calibration set $\{o_{m+1}, \ldots, o_{n+1}\}$ and only makes use of the proper training set. Hence changing the permutation of a calibration set does not effect the nonconformity scores and remains the same. Therefore the proposed (non)conformity scores are invariant of permutation of $\{o_{m+1}, \ldots, o_{n+1}\}$. 
\end{proof}

% We define $p$-value in the next section by employing
%(non)conformity measure as defined above. 
\begin{algorithm}[h]
\SetAlgoLined
%profile(u, \forall u \in U)
\KwIn{$O, ~ target~ user~ u_j, ~ O_j, ~ \varepsilon$}
\KwOut{Recommendation set ($\Gamma^{\varepsilon}$)}

split $O_j$ into two sets $O^{t}_{j}$ and
$O^{c}_{j}$\;
$\Gamma^{\varepsilon} \leftarrow \emptyset $ \;

\For{each $o_h$ in $O^c_j$} {
	Compute $\alpha_h$ using any of the (non)conformity measures\;
}
\For {each $o$ $\in O\setminus O_j$} {
	Compute (non)conformity score of an item $o$\;
	Compute $p(o)$ using Equation \ref{eq-p1} or  \ref{eq-p2}\;
	
	\lIf { $p(o) > \varepsilon$ } {
		$\Gamma^{\varepsilon} \leftarrow \Gamma^{\varepsilon} \cup \{o\}$
	}
}
\caption{Inductive Conformal Recommender Systems.}
\label{algo:ICR}
\end{algorithm}

\subsection{p-value and Recommendation Set}
\label{sec:pvalue}
Let $\alpha_h$ be the conformity or nonconformity value of an item $o_h$. For nonconformity measures, the proportion of examples having a nonconformity value greater than the new example defines the $p$-value,  
\begin{equation}
% \scriptsize
p(o_{n+1}) = \frac{\Big\vert\big\{ h \big\vert  m+1\le h \le n+1,
	\alpha_h \ge  \alpha_{n+1}\big\}\Big\vert}{l+1}.
\label{eq-p1}
\end{equation}
In the case of conformity measure, we define it as the proportion of examples having conformity value less than the new example,
\begin{equation}
% \scriptsize
p(o_{n+1}) = \frac{\Big\vert\big\{ h \big\vert  m+1\le h \le n+1,
	\alpha_h \le  \alpha_{n+1}\big\}\Big\vert}{l+1}.
\label{eq-p2}
\end{equation}
%\subsection{Recommendation Set }
%\label{sec:rec}
For a target user $u_j$, the recommendation set is then constructed by computing the $p$-value for every unused item. All the items whose $p$-value is 	greater than the predetermined significance level $\varepsilon$ will form a recommendation region $\Gamma^\varepsilon$.
\begin{equation*}
\Gamma^\varepsilon = \{o\mid p(o)> \varepsilon\}.
\end{equation*}

\subsection{Algorithm } %\textcolor{red}{Do you really need it?}
\label{Algo_Flow}
In this section, we describe the algorithm by using the concepts defined in the previous sections. Algorithm~\ref{algo:ICR} outlines the main flow of the proposed method. At first, we divide the dataset into a proper training set and calibration set. Next, we compute every item's nonconformity value in the calibration set and for every candidate item. We then compute the $p$-value for every candidate item and determine the recommendation set. The flowchart of the proposed algorithm is shown in Figure~\ref{icrs_flow}. 
\begin{figure}[ht!]
\adjustbox{max width=\linewidth}{
	\includegraphics[scale=1]{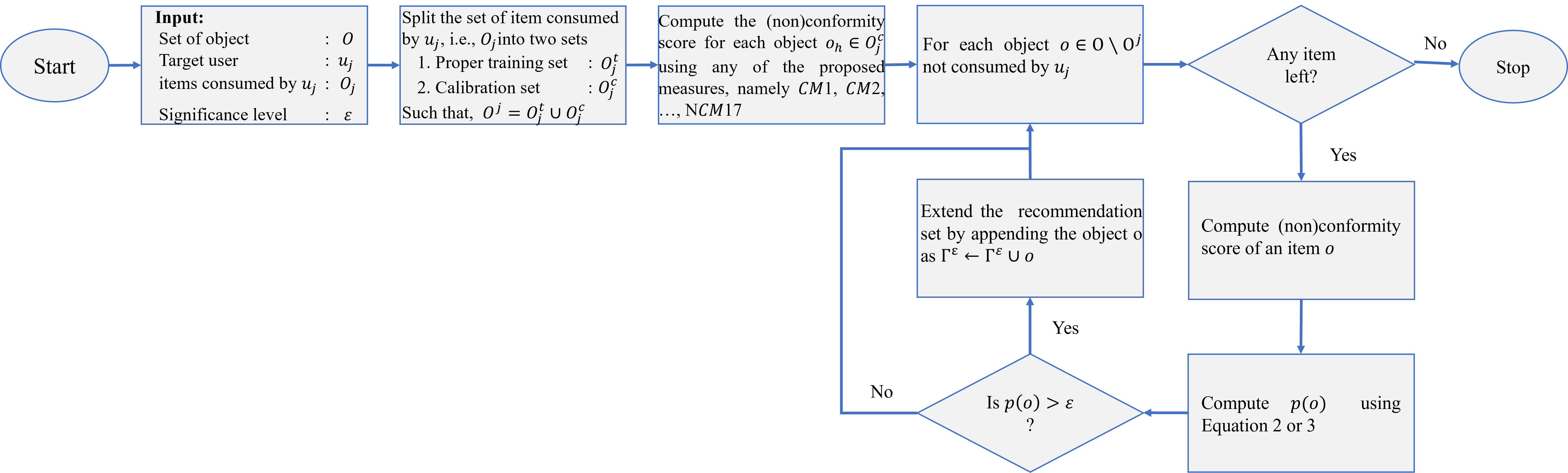}
}
\caption{Inductive Conformal Recommender System.} 
\label{icrs_flow}
\end{figure}
\begin{Example}
We consider the precedence statistics given in Example 3.  %in the previous examples.
We divide the target user $u_1$ profile into $O_1^t = \{o_1, o_3, o_5\}$ and $O_1^c = \{o_7, o_9\}$. Let us compute the nonconformity values, and $p$-value with respect to $o_2$ and the conformity measure \texttt{CM1}. The procedure is similar for other (non)conformity measures also. 

\[
CM1(o_7)  = \frac{Support(o_7)}{30} \times top1\big(PP(o_1\mid o_7), PP(o_3\mid o_7), PP(o_5\mid o_7)\big) \\
= \frac{15}{30}\times\frac{7}{15} = 0.23.
\]
Similarly, $CM1(o_9)= 0.3$ and $CM1(o_2) = 0.3$. Hence,
\[
p(o_2) = \frac{\big\lvert\{o_h\mid o_h\in \big\{O_1^c\cup\{o_2\}\big\}\wedge CM1(o_h)\le CM1(o_2) \}\big\rvert }{\lvert O_1^c\cup\{o_2\}\rvert} = \frac{3}{3} = 1.
\] 
\noindent We can compute the $p$-value for all the other candidate items and include those items whose $p$-value is greater than $\varepsilon$ in the recommendation set.
\end{Example}

\subsection{Validity and Efficiency}
\label{sec:validity}
We have already shown that the (non)conformity measures defined above satisfy	the invariant property in Lemma~\ref{lem:1}. Hence, following the line of argument given by Vovk et al.~\cite{vovk3}, it is easy to see that our proposed method ICRS satisfies the validity property.
\begin{Lemma}
\label{lem1}
If objects $o_{m+1}, o_{m+2}, \ldots, o_{n+1}$ are independently and identically distributed (i.i.d.) in terms of their precedence relations with individual items in the history, then the probability of error that $o_{n+1}\notin \Gamma^\varepsilon(o_1, o_2, \ldots, o_m)$
will not exceed $\varepsilon \in [0,1]$ i.e., $P(P(o_{n+1})\le \varepsilon)\le \varepsilon.$
\end{Lemma}%
\begin{proof}
An error occurs when $P(o_{n+1})\le \varepsilon$. 
That is, when $\alpha_{n+1}$  is among the $\lfloor \varepsilon(l+1)\rfloor$ largest elements of the set $\{\alpha_{m+1}, \alpha_{m+2}, \ldots, \alpha_{n+1}\}$. When all the objects are in i.i.d in terms of precedence relations with the set of items consumed by an user, all permutations of the set $\{\alpha_{m+1}, \ldots, \alpha_{n+1}\}$ are equiprobable. 
Thus, the probability that $\alpha_{n+1}$ is among the $\lfloor\varepsilon(l+1)\rfloor$ largest elements does not exceed $\varepsilon$, which is therefore the probability of error.
\end{proof}
In addition to satisfying the validity property, it is desirable to have an efficient recommendation set. In the conformal framework setting, a narrow set with higher confidence is more efficient. 	We empirically analyze the validity and efficiency properties in	Section~\ref{sec:Exp}.
%%%%%%%%%%%%%%%%%%%%%%%%%%%%%%%%%%%%%%%%%%%%%%%%%%%%%%%%%%%%%%%%%%%%%%%

\subsection{Time Complexity Analysis}
\label{timeCAnalysis}
In this section, we analyze the time complexity of the proposed method against transductive conformal recommender systems~\cite{kagita2017conformal} and the underlying precedence mining based algorithm~\cite{ParameswaranKBG10}. For simplicity, we assume that the 	calibration set size is the same as that of candidate-set ($\lvert O^{candidates}_j\rvert$) in the Conformal Recommender System~\cite{kagita2017conformal}.  We know that $m$ is the size of the proper training set, and $l$ is the calibration set size. 	Let $n_c$ be the number of candidate items i.e., $n_c = n_i - n$. 	Since the complexity of measuring nonconformity scores varies from measure to measure, we assume it to be $O(t)$. With $O(t)$ as the complexity of nonconformity measure, the inductive conformal predictor takes $O((l + n_c)t)$ time 	complexity to determine all the required $p$-values and make recommendations. On the other hand, transductive conformal recommender systems take $O(n_c lmt)$ complexity with $O(t)$ as the nonconformity measure's complexity. Kagita et al.~\cite{kagita2017conformal} reduce it to $O(n_c lm)$ using the relation between the score and precedence probability, but it is higher than the inductive conformal recommender system. The complexity of the precedence mining based recommender system is $O(n_cn)$.

\begin{table}[ht!]
\centering
\caption{Summary of experimental datasets.}
\renewcommand{\arraystretch}{1.2}
\label{tab:dataset}
\begin{tabular}{|l||l|l|l|}
	\hline
	{\bf Dataset} & {\bf Users} & {\bf Items}  & {\bf Records}\\ \hline
	%MovieLens-latest-small & 707 & 8,553 & 100,024\\ \hline
	Personality-2018 & 1820  & 35196 & 1,028,752\\ \hline
	Flixsters & 20,618 & 28,331 &  1,048,575\\ \hline
	%MovieLens 100K & 943 & 1,682 & 100,000\\ \hline
	%MovieLens 1M & 6,040 & 3,952 & 1,000,209\\ \hline
	MovieLens 10M & 71,567 & 10,681& 10,000,054\\ \hline
	MovieLens 20M & 138,494 & 26,745 & 20,000,262 \\ \hline
	MovieLens 25M & 162,000 & 62,000 & 25,000,096 \\ \hline
	MovieLens-Latest-V1 & 229,061 & 26,780 &  21,063,128\\ \hline
	MovieLens-Latest-V2 & 280,000 & 58,000 &  27,753, 445\\ \hline
	%Lastfm & 1,892 & 12,523 & 92,800 \\ \hline
	%Foursquare & 2,321 & 5,596 & 194,108\\ \hline
	%Gowalla & 10,162 & 24,237 & 456,905\\ \hline
\end{tabular}
\end{table}

\section{Empirical Study}
\label{sec:Exp}
In this section, we empirically evaluate the efficacy of proposed \emph{Inductive Conformal Recommender System (ICRS)}. We provide an in-depth quantitative evaluation with regard to the prediction accuracy and running time on seven real-world datasets of varying size. The characteristics of these datasets are reported in Table~\ref{tab:dataset}. %The datasets are preprocessed to remove all the users with less than $20$ ratings. 	
In all our experiments, we converted the multi-class (different ratings) datasets into one class by setting a threshold to $0$. The prediction accuracy of the comparing algorithms are evaluated based on the ranking-based performance metrics that is \emph{Average Precision (AP),  Area Under Curve (AUC), Normalized Discounted Cumulative Gain (NDCG)} and \emph{Reverse Reciprocal (RR)}. We also evaluated the performance based on top-K recommendation metrics that is \emph{Precion@K, Recall@K} and \emph{F1@K}~\cite{Huayu2016}. 
We compared our proposed method ICRS with the underlying Precedence Mining Model~\cite{ParameswaranKBG10} and the Conformal Recommender Systems (CRS-max and CRS-med)~\cite{kagita2017conformal}. In ICRS, to fine-tune the values of parameters $n$ and $k$, we experimented with different combinations and selected $n$ to be $30\%$ of the profile and $k$ to be $30\%$ and the remaining $40\%$ is the test data.  All the results reported here are the average of $500$ randomly selected instances. We use a notation $ICRS<x>$ to denote an inductive conformal recommender system that uses (non)conformity measure $x$. For example, $ICRS1$ uses conformity measure $1$ (CM1).

The remainder of the section is structured as follows. In Section~\ref{Validity_Efficiency}, we report the experimental evaluation of the validity and efficiency of the proposed methods. Section~\ref{ComparativeAnalysis} report comparative experimental results in terms of \textit{ranking-based metrics, top-k recommendation metrics}  and \textit{execution time}.

%This section reports the experimental evaluation of the validity and efficiency of the proposed \emph{Inductive Conformal Recommender System (ICRS)} on seven real-world datasets.  
%It also offers a comparative study against the underlying Precedence Mining Model~\cite{ParameswaranKBG10} and the Conformal Recommender Systems (CRS-max and CRS-med)~\cite{kagita2017conformal}.  
%We perform comparative analysis using different standard evaluation measures such as \emph{ Average Precision (AP),  Area Under Curve (AUC), NDCG, Reverse Reciprocal (RR), F1-measure, precision, recall,} and \emph{time}. 
%Table~\ref{tab:dataset} summarizes the characteristics of the datasets used in the experimentation. The datasets are preprocessed to remove all the users with less than $20$ ratings. 	In all our experiments, we converted the multi-class (different ratings) datasets into one class by setting a threshold to $0$. 

%To fine-tune the values of parameters $n$ and $k$, we experimented with different combinations and selected $n$ to be $30\%$ of the profile and $k$ to be $30\%$ and the remaining $40\%$ is the test data.  All the results reported here are the average of $100$ randomly selected instances. We use a notation $ICRS<x>$ to denote an inductive conformal recommender system that uses (non)conformity measure $x$. For example, $ICRS1$ uses conformity measure $1$ (CM1). 

\subsection{Validity and Efficiency}
\label{Validity_Efficiency}
This subsection empirically evaluates the validity and efficiency of the proposed approach. We adapt the definitions of validity and efficiency given by Kagita et al.~\cite{kagita2017conformal}.  Figure~\ref{fig:ICRS_err1} and Figure~\ref{fig:ICRS_err2} shows the validity  and efficiency of the proposed approach respectively, over seven different datasets. We report the validity and efficiency related to \emph{ICRS1, ICRS3, ICRS7, ICRS11} and \emph{ICRS15}\footnote{ICRS3, ICRS7, ICRS11, and ICRS15 use the median strategy. Similar results have been observed for other strategies also.}.  It can be seen from Figure~\ref{fig:ICRS_err1} that the error is proportional to $\varepsilon$ and in the relative bound of $\varepsilon$. 
%%except for the \textit{Lastfm}, \textit{Foursquare}, and \textit{Gowalla} datasets. 
%The detailed analysis of these datasets revealed that the users repeatedly utilized some items, which violates our assumption of strict precedence relation. We believe this could be the reason for the violation of validity property in those datasets.
%
Figure~\ref{fig:ICRS_err2} reports the error related to efficiency. It can be seen from the figures that even for smaller values of $\varepsilon$, most of the irrelevant items are filtered out hence, resulting in a small error.  We also observed that, for higher values of $\varepsilon$, the recommendation set is more informative for all the strategies. 
%We also noticed from the graph that \emph{ICRS11} is resulting in less informative recommendations than others. We observe a similar result in the case of \emph{ICRS10, ICRS12}, and \emph{ICRS13}. 

\begin{figure}[h!]
\centering
\begin{subfigure}{0.24\linewidth}
	\centering
	\includegraphics[width=\textwidth, height = 3.1cm]{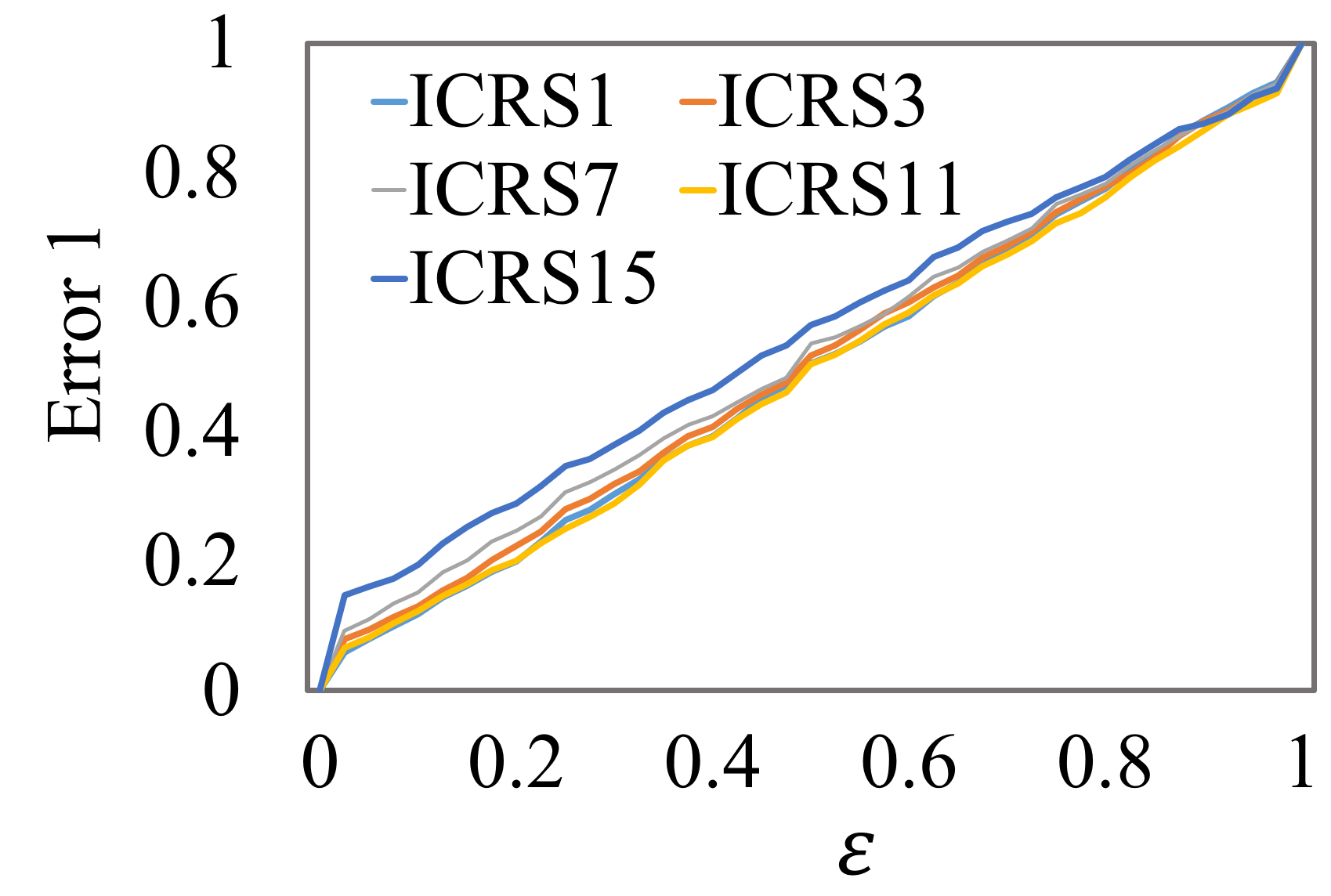}
	\caption{Personality-2018}
	%\label{fig:sub2}
\end{subfigure}
\begin{subfigure}{0.24\linewidth}
	\centering
	\includegraphics[width=\textwidth, height = 3.1cm]{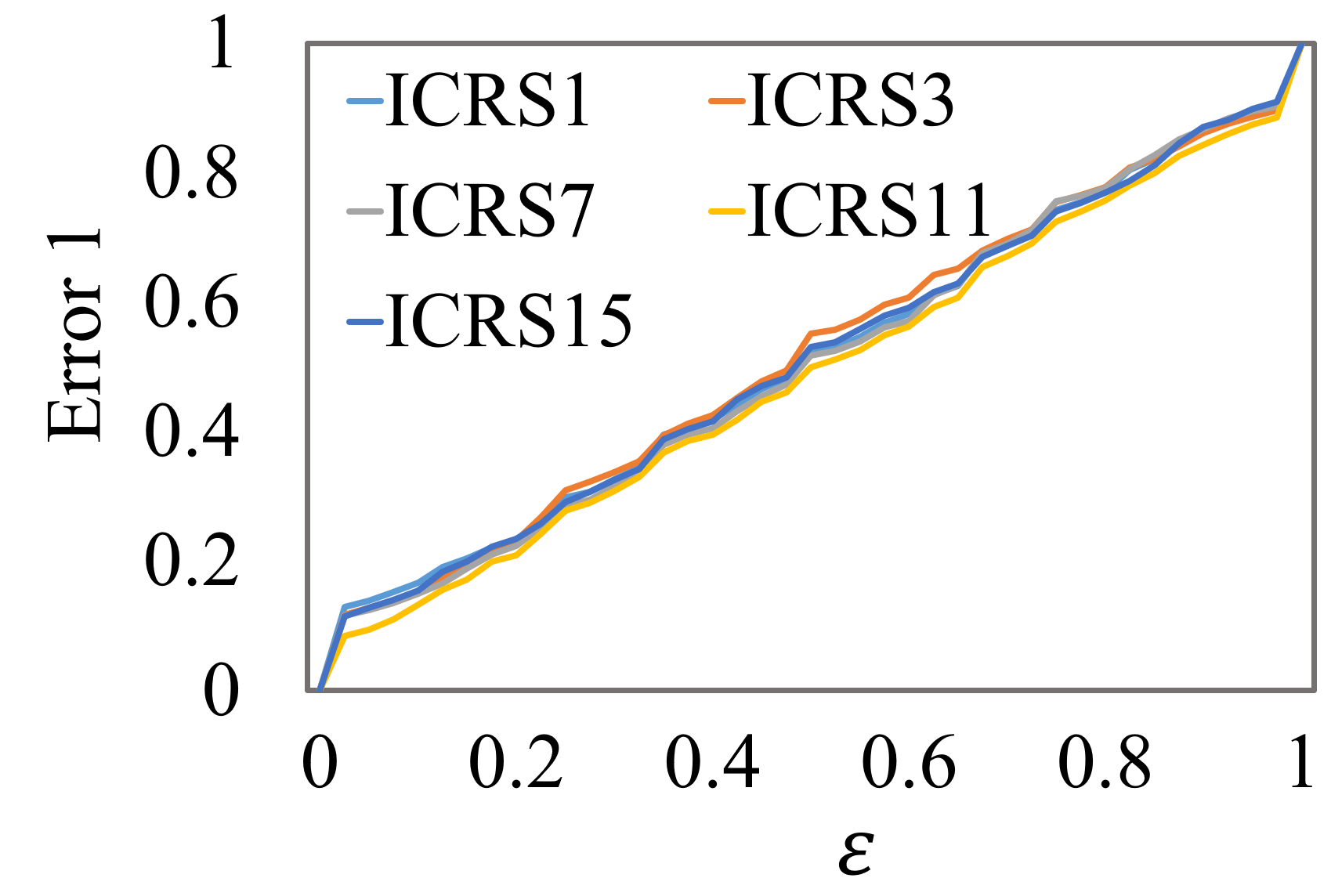}
	\caption{Flixsters}
	%\label{fig:sub4}
\end{subfigure}
\begin{subfigure}{0.24\linewidth}
	\centering
	\includegraphics[width=\textwidth, height = 3.1cm]{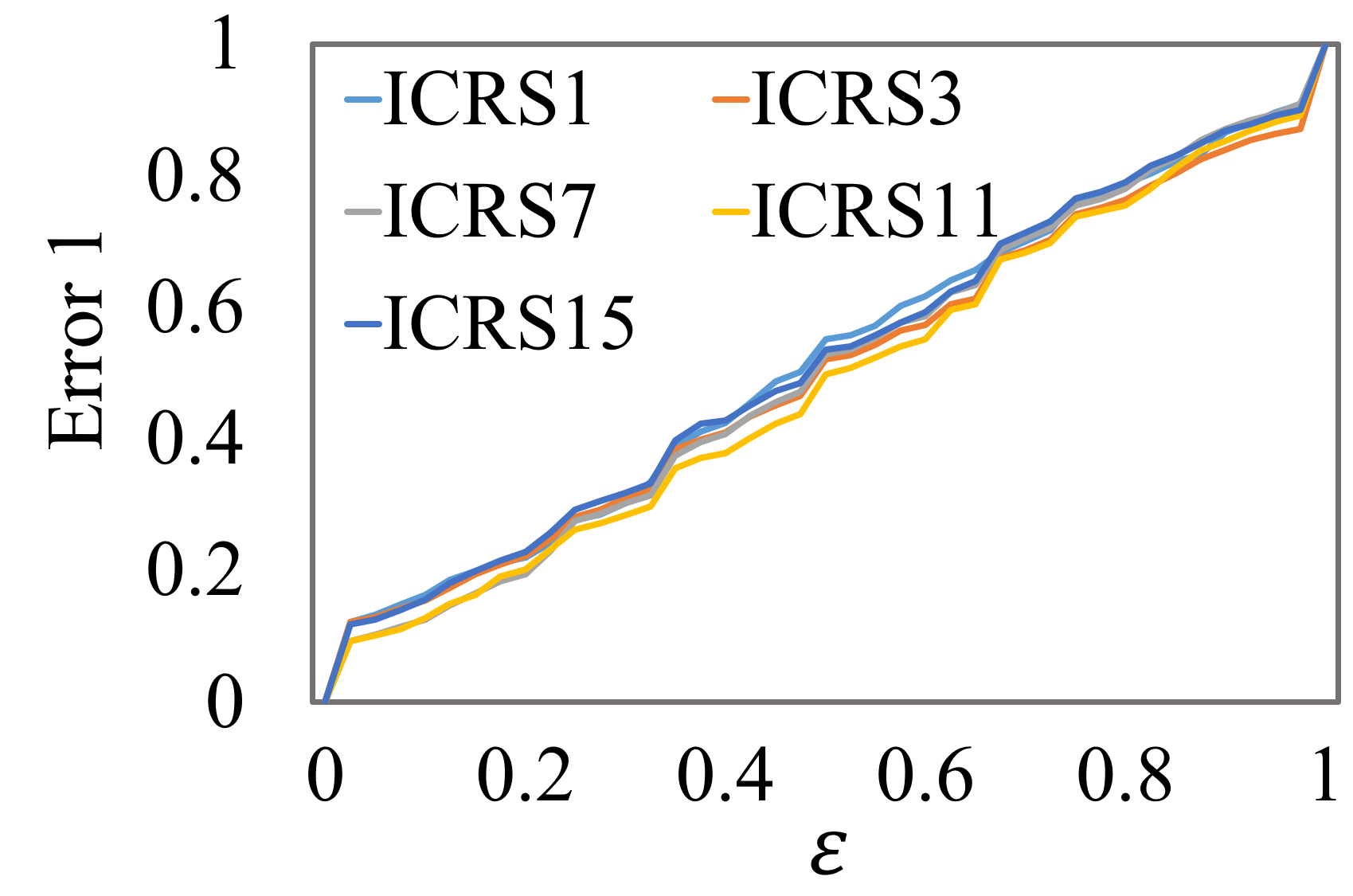}
	\caption{MovieLens 10M}
	%\label{fig:sub5}%
\end{subfigure}
\begin{subfigure}{0.24\linewidth}
	\centering
	\includegraphics[width=\textwidth, height = 3.1cm]{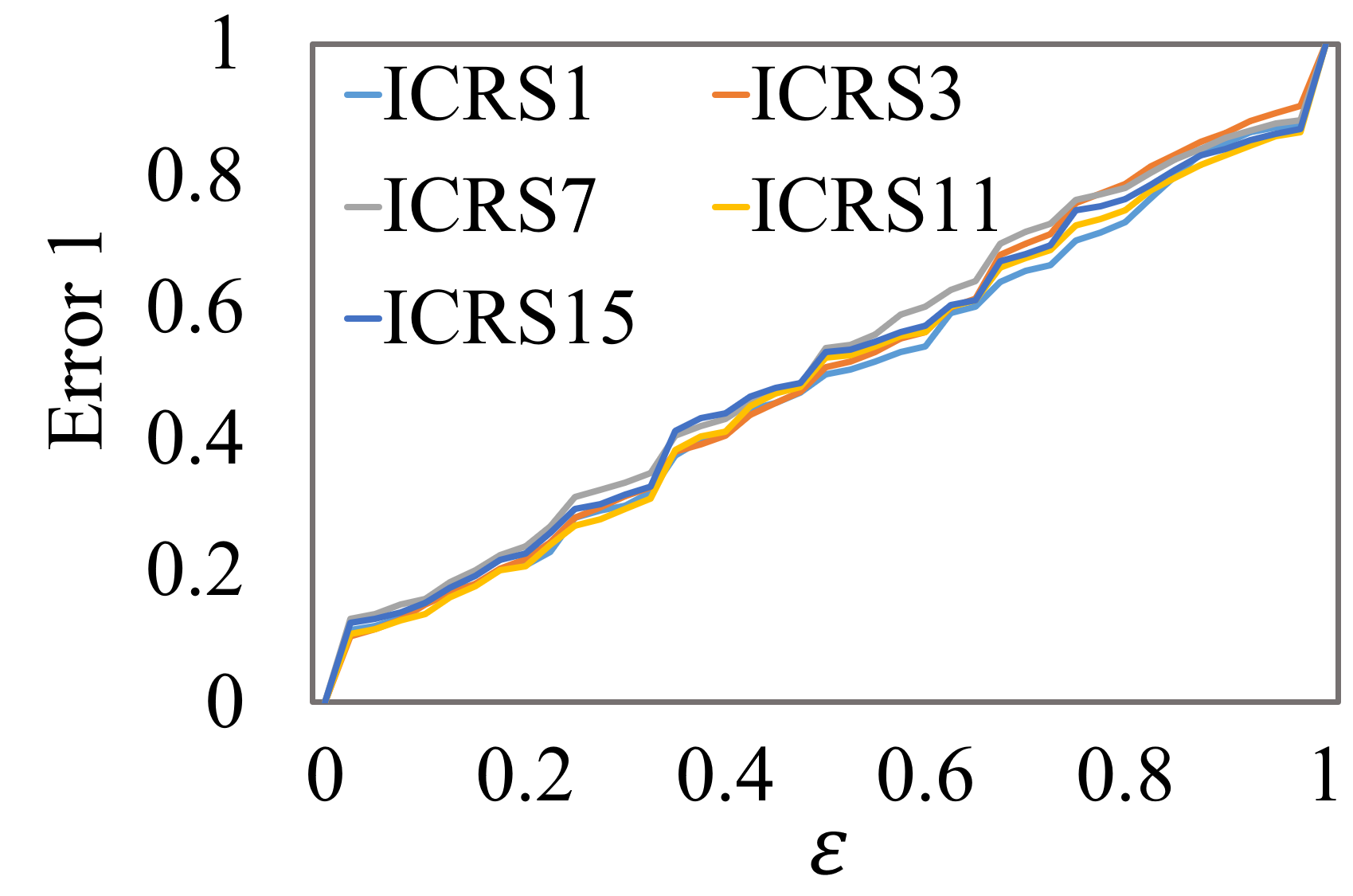}
	\caption{MovieLens 20M}
	%\label{fig:sub6}
\end{subfigure}\\
\vspace{0.2cm}
\begin{subfigure}{0.24\linewidth}
	\centering
	\includegraphics[width=\textwidth, height = 3.1cm]{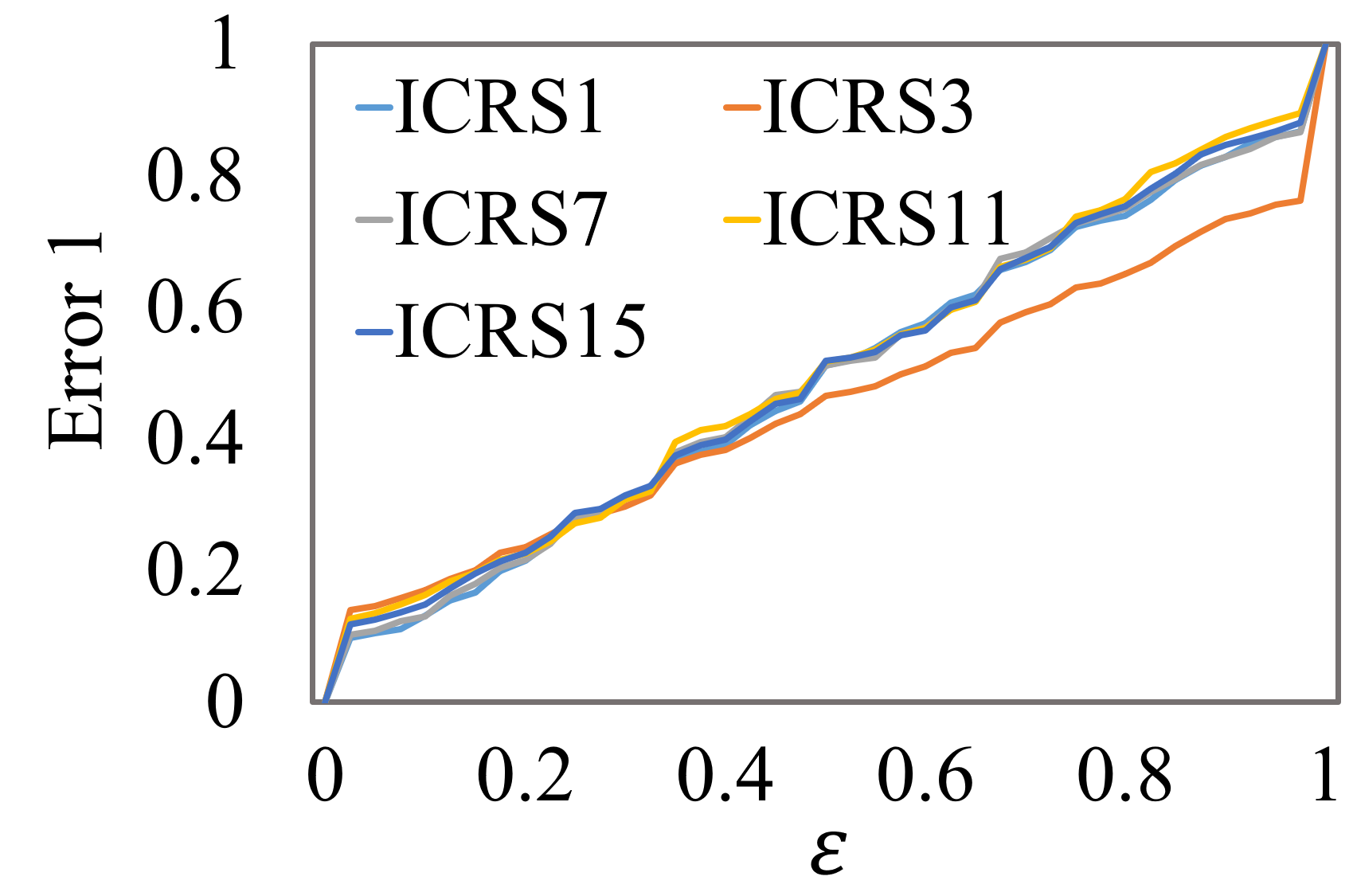}
	\caption{MovieLens 25M}
	%\label{fig:sub4}
\end{subfigure}
\begin{subfigure}{0.24\linewidth}
	\centering
	\includegraphics[width=\textwidth, height = 3.1cm]{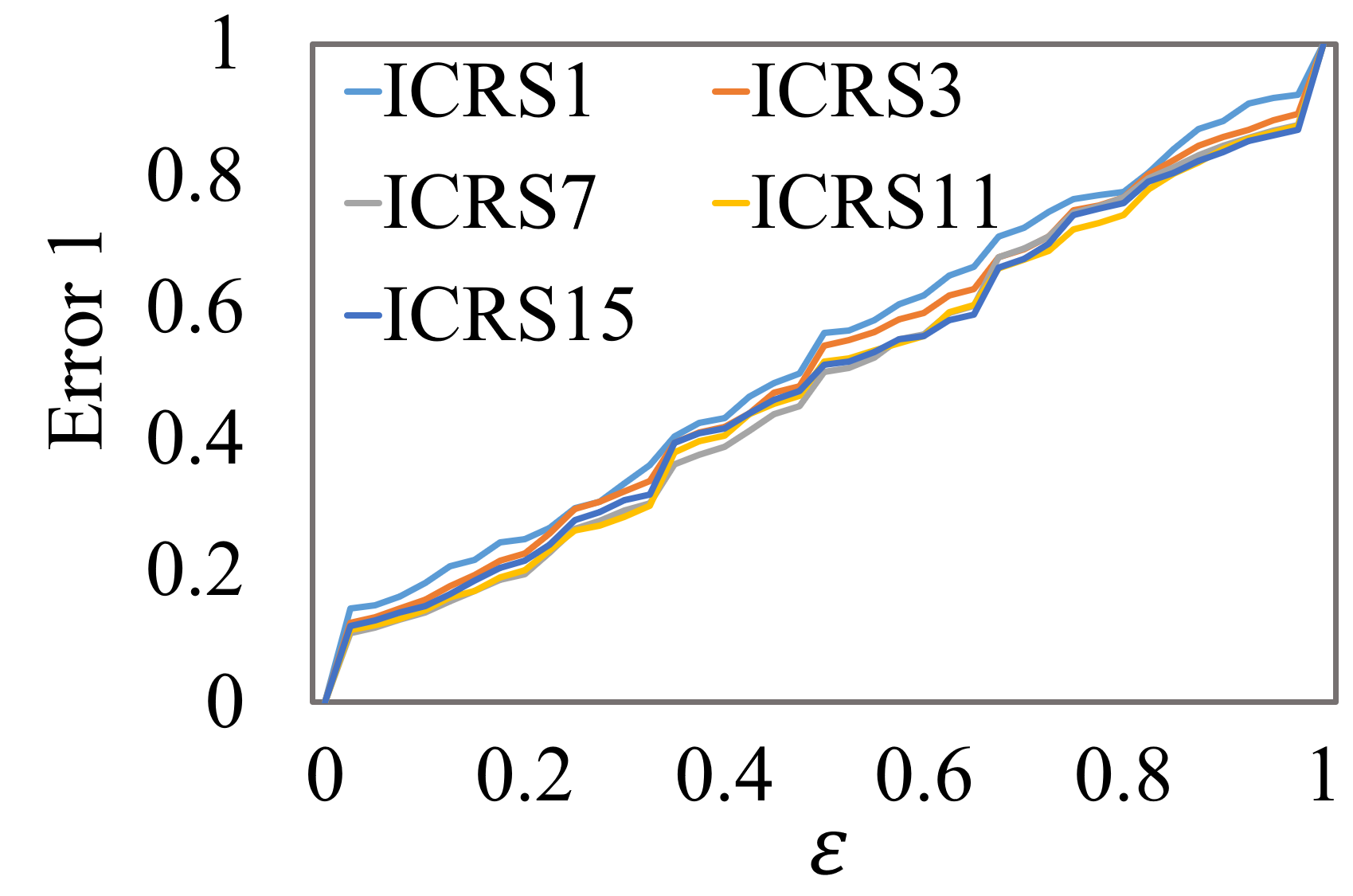}
	\caption{MovieLens-latest-V1}
	%\label{fig:sub5}%
\end{subfigure}
\begin{subfigure}{0.24\linewidth}
	\centering
	\includegraphics[width=\textwidth, height = 3.1cm]{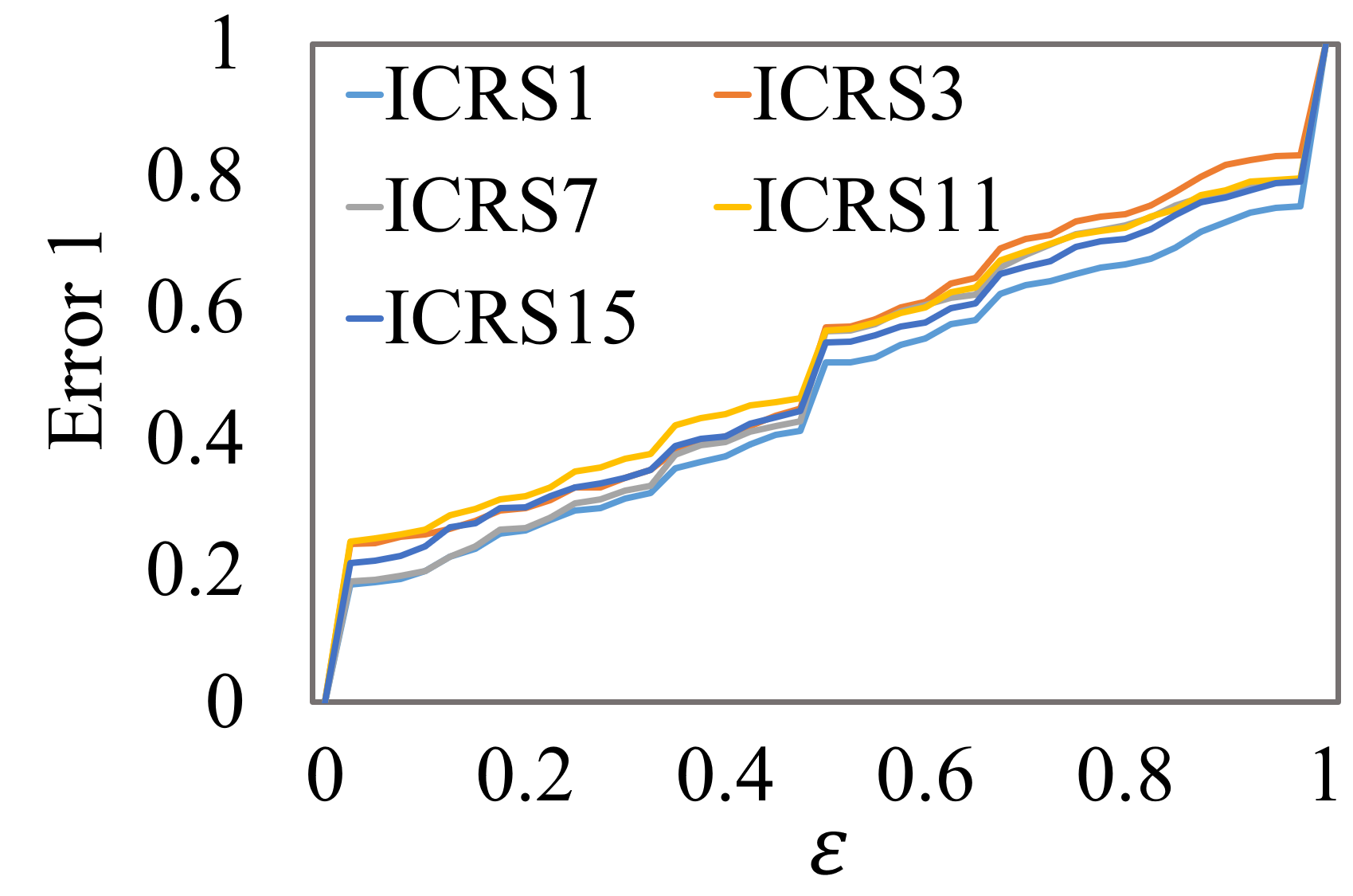}
	\caption{MovieLens-latest-V2}
	%\label{fig:sub6}
\end{subfigure}
%\begin{subfigure}{0.24\linewidth}
\hspace{0.24\linewidth}
%\end{subfigure}%
\caption{Evaluation of recommendation validity for different datasets}
\label{fig:ICRS_err1}
\end{figure} 
\begin{figure*}[ht!]
\centering
\begin{subfigure}[c]{0.24\linewidth}
	\centering
	\includegraphics[width=\textwidth, height = 3.1cm]{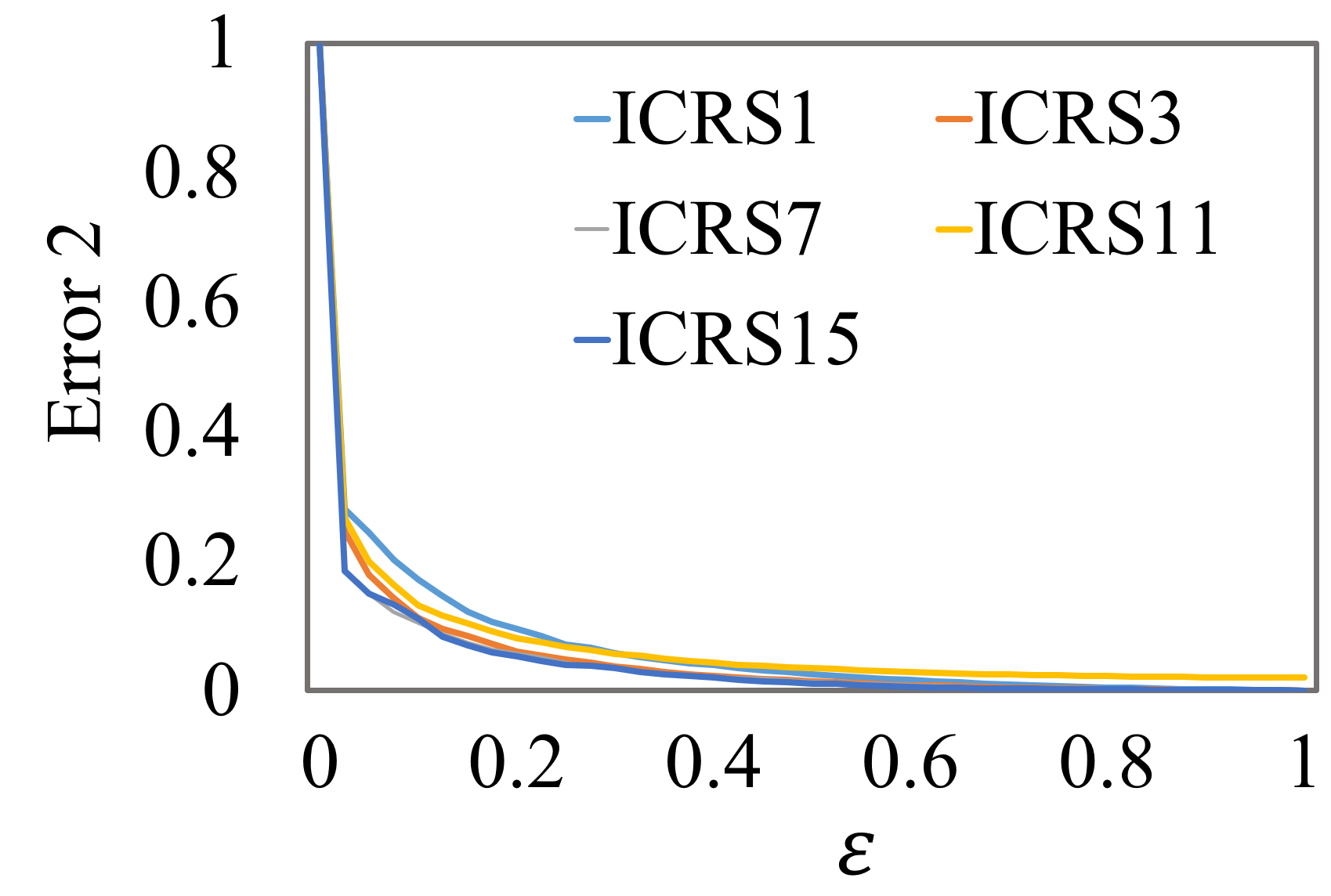}
	\caption{Personality-2018}
	%\label{fig:sub2}
\end{subfigure}
\begin{subfigure}[c]{0.24\linewidth}
	\centering
	\includegraphics[width=\textwidth, height = 3.1cm]{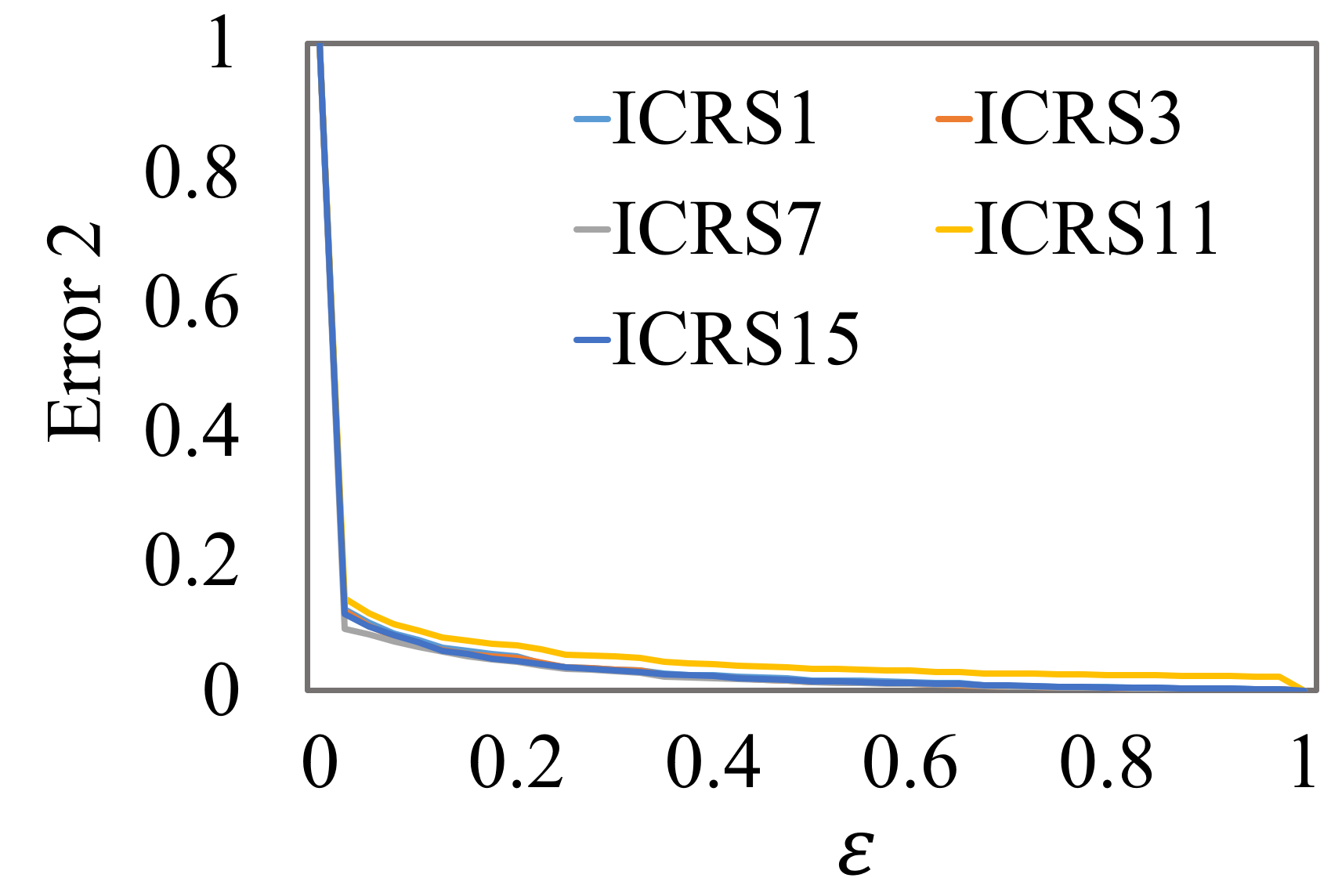}
	\caption{Flixsters}
	%\label{fig:sub4}
\end{subfigure}
\begin{subfigure}[c]{0.24\linewidth}
	\centering
	\includegraphics[width=\textwidth, height = 3.1cm]{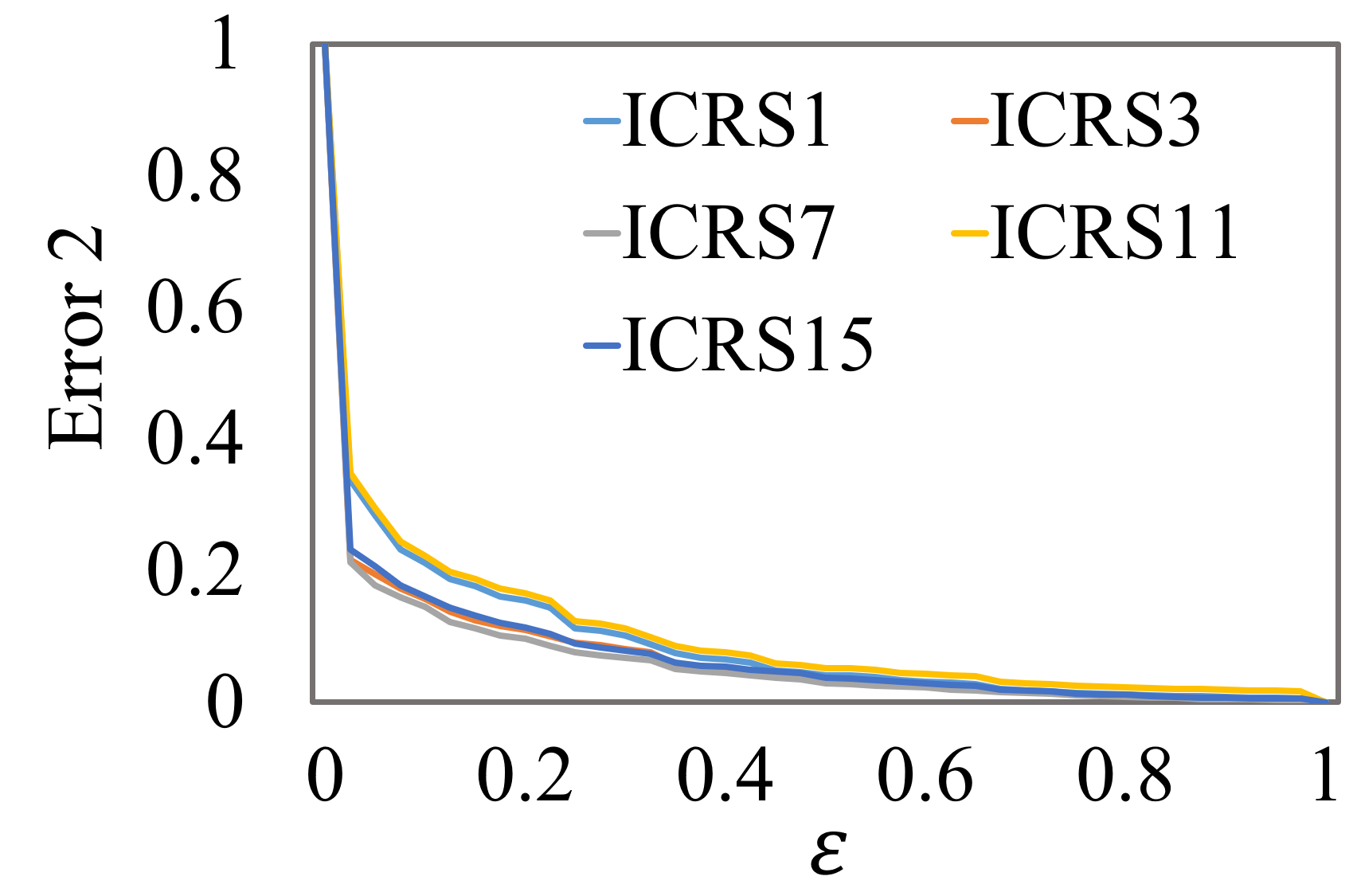}
	\caption{MovieLens 10M}
	%\label{fig:sub5}%
\end{subfigure}
\begin{subfigure}[c]{0.24\linewidth}
	\centering
	\includegraphics[width=\textwidth, height = 3.1cm]{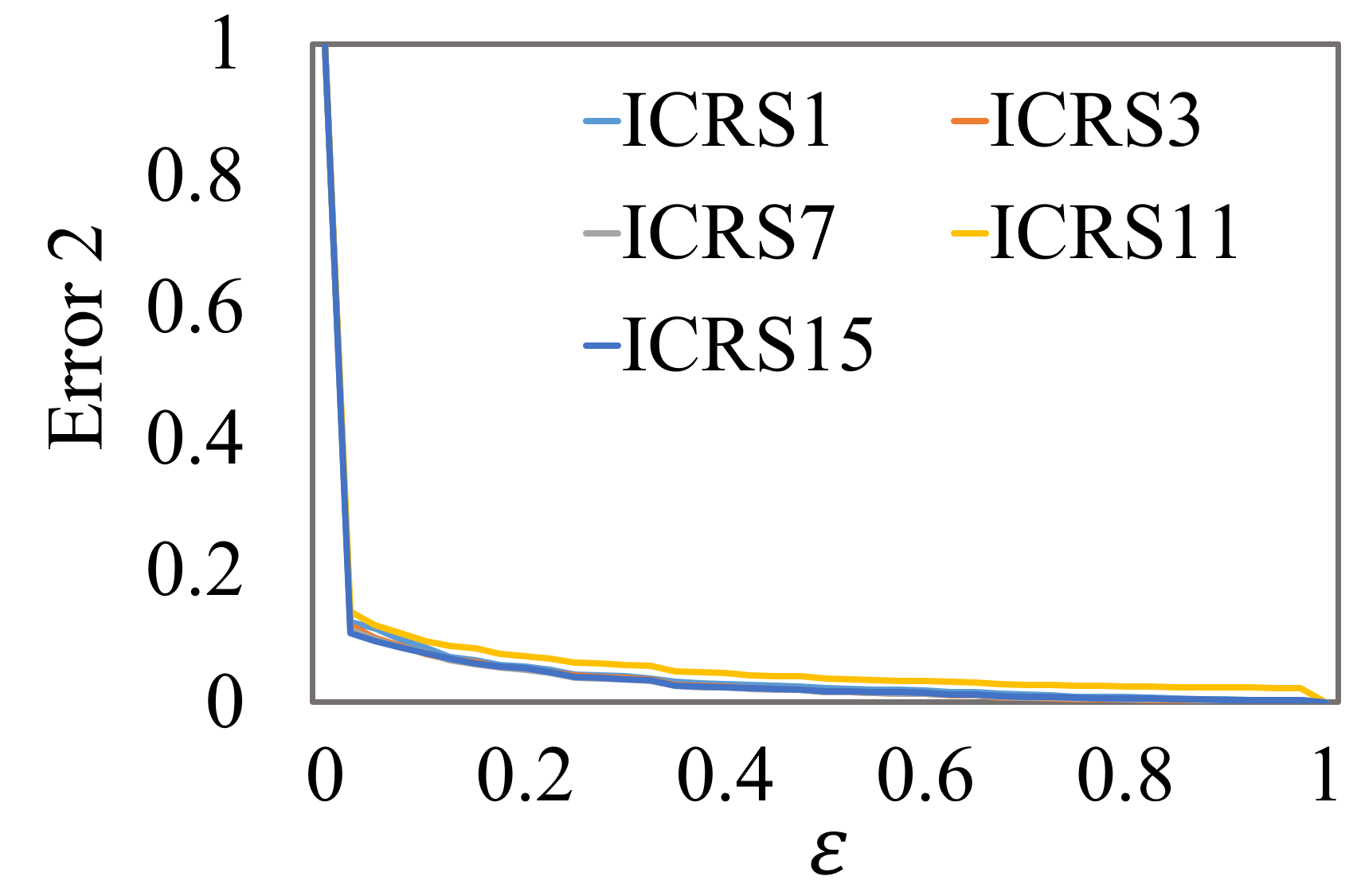}
	\caption{MovieLens 20M}
	%\label{fig:sub6}
\end{subfigure}\\
\vspace{0.2cm}
\begin{subfigure}[c]{0.24\linewidth}
	%\centering
	\includegraphics[width=\textwidth, height = 3.1cm]{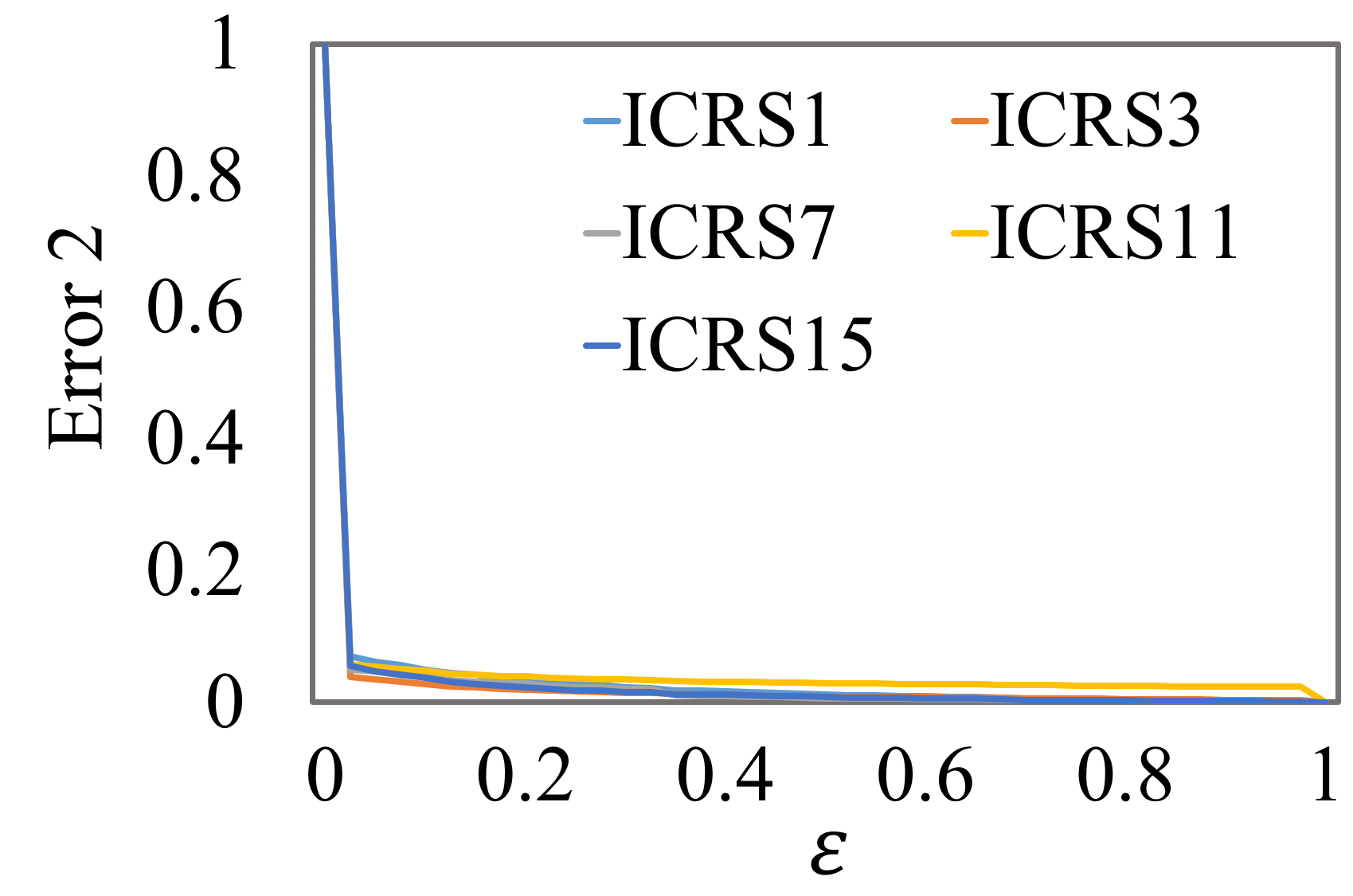}
	\caption{MovieLens 25M}
	%\label{fig:sub4}
\end{subfigure}
\begin{subfigure}[c]{0.24\linewidth}
	%\centering
	\includegraphics[width=\textwidth, height = 3.1cm]{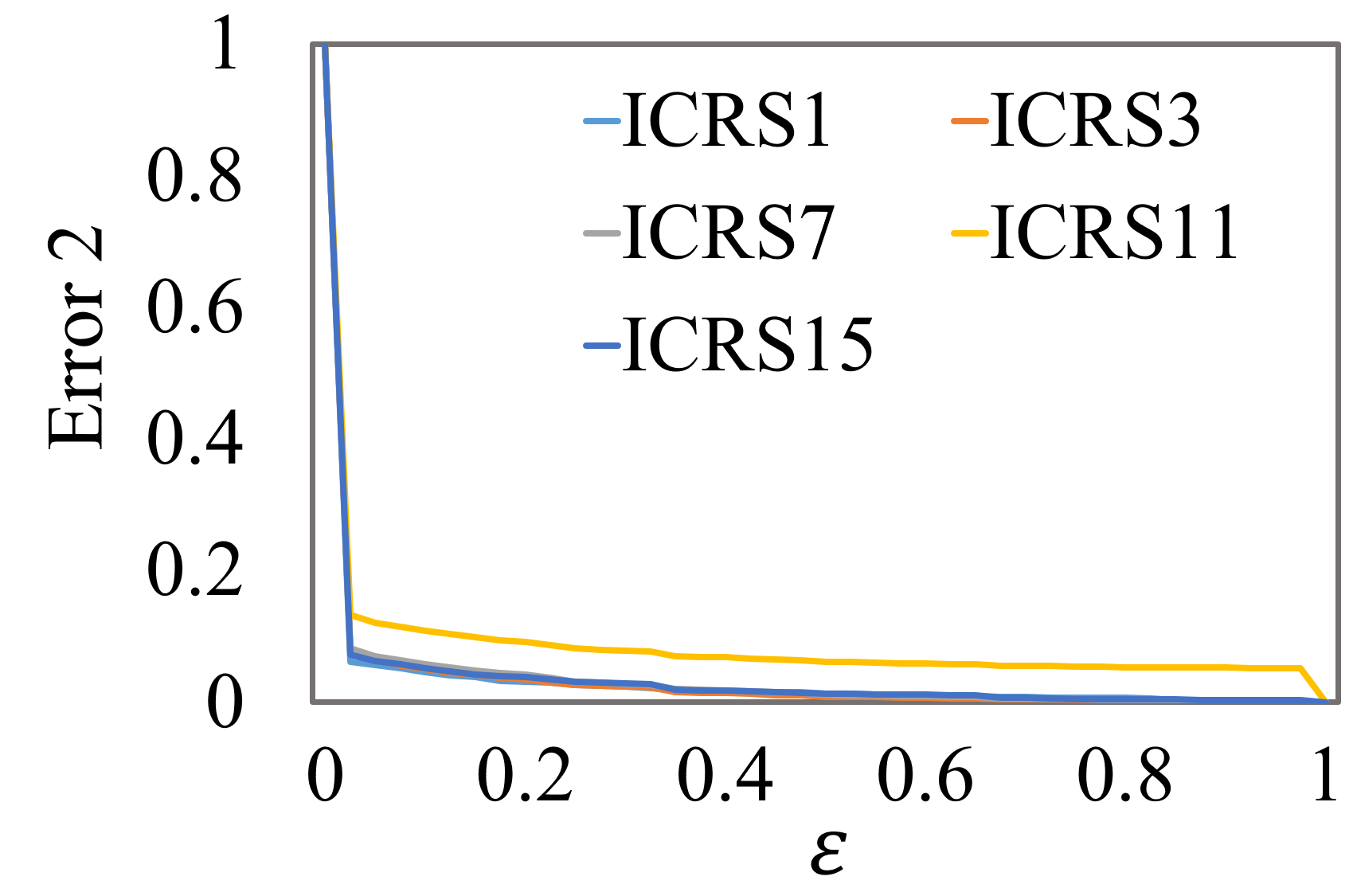}
	\caption{MovieLens-latest-V1}
	%\label{fig:sub5}%
\end{subfigure}
\begin{subfigure}[c]{0.24\linewidth}
	%\centering
	\includegraphics[width=\textwidth, height = 3.1cm]{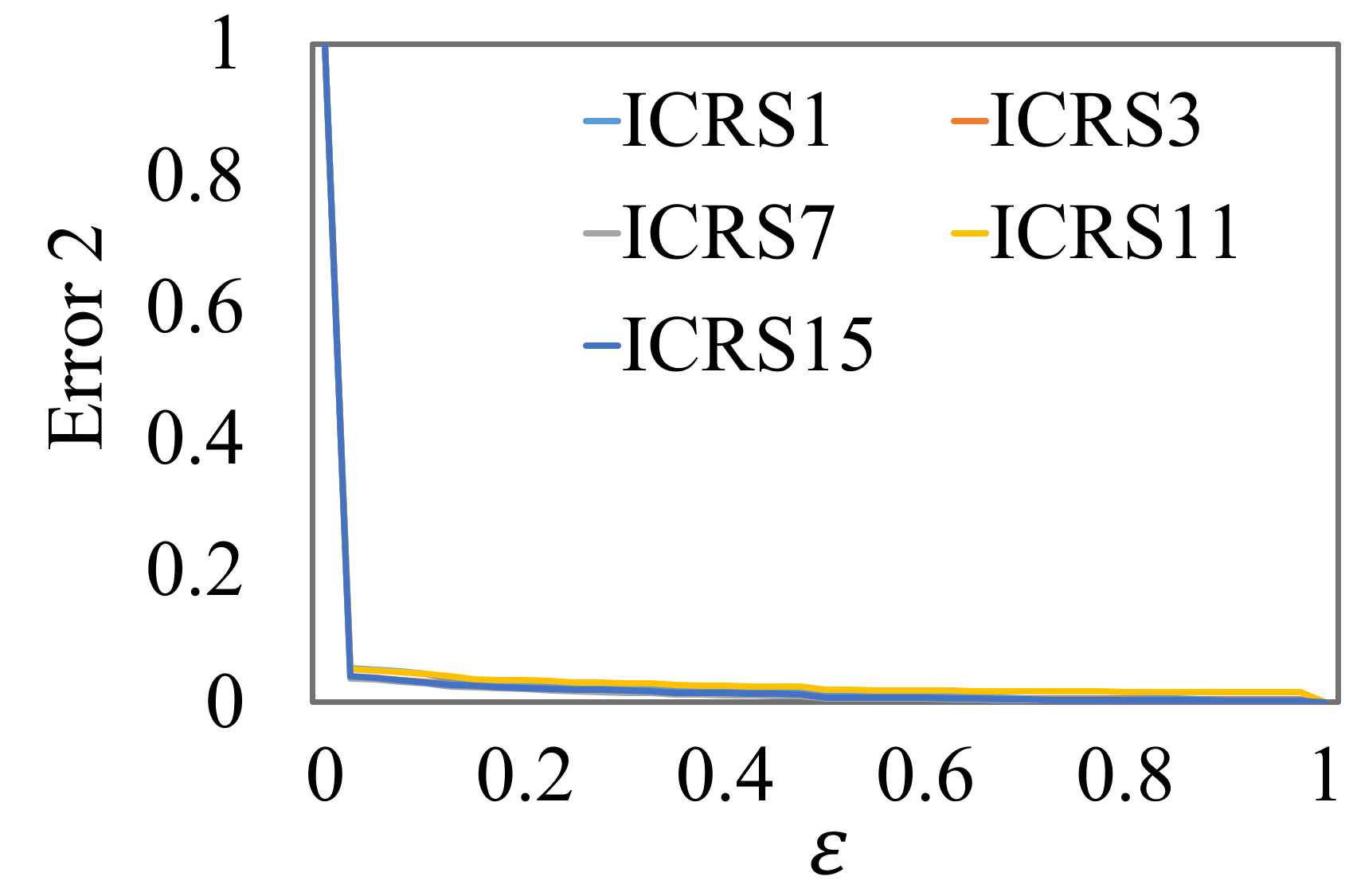}
	\caption{MovieLens-latest-V2}
	%\label{fig:sub6}
\end{subfigure}
\hspace{0.24\linewidth}
\caption{Efficiency of recommendation for different datasets}
\label{fig:ICRS_err2}
\end{figure*}

%\subsection{Comparative Analysis}
\subsection{Comparative Analysis}
\label{ComparativeAnalysis}

% We carried out a second set of experiments to demonstrate the superiority of the proposed ICRS. The results are compared against the underlying precedence mining model and two conformal recommendation algorithms, namely, CRS-max and CRS-med. Table 1 gives the findings related to 10 datasets and four evaluation measures. It can be seen that most of the inductive methods are comparable to conformal techniques and sometimes better than them. Furthermore, we also compared the execution times of the different approaches in Figure~\ref{fig:ICRS_time}. 	   
%In this section, we carried out a second set of experiments to demonstrate the superiority of the proposed ICRS.
In this section, we carried out experiments to demonstrate that the proposed methods achieve comparable results with significantly reduced execution times. 
%The results are compared against the underlying precedence mining model and two conformal recommendation algorithms, \emph{CRS-max} and \emph{CRS-med}. 
Table~\ref{tab:AP_AUC_NDCG_RR} gives the findings related to ranking-based evaluation measures over seven datasets. Each result is composed of \textit{mean} and \textit{rank}. The rank reflects relative performance of an algorithm over a dataset for a given evaluation measure. In the case of ties, we have assigned the average rank. Furthermore, the entries in boldface highlight best results among all the algorithms being compared.

To carry out comparative analysis in more well-founded ways, we employed \textit{Friedman test} which is widely-accepted as the favorable statistical  test for comparing more than two algorithms over multiple data sets~\cite{demvsar2006statistical}. For each evaluation criterion, \textit{Friedman statistics} $F_F$ and the corresponding critical value are reported in Table \ref{ffTest}. It can be observed that at significance level $\alpha = 0.05$, Friedman test rejects the null hypothesis of “equal” performance for each evaluation metric. This leads to the use of post-hoc tests to assess the pairwise differences between two algorithms within a multiple comparison test. We use the Nemenyi test to check whether the proposed methods achieves a competitive performance against the algorithms being compared~\cite{demvsar2006statistical}.  The performance of two algorithms is significantly different if the corresponding average ranks differ by at least the critical difference $CD = q_{\alpha}\sqrt{ \frac{\mathcal{K}(\mathcal{K}+1)}{6\mathcal{N}}}$, where the value $q_{\alpha}$ is based on the Studentized range statistic divided by $\sqrt{2}$. For Nemenyi test with $\mathcal{K}=20$, we have $q_{\alpha}= 3.5438$ at significance level $\alpha = 0.05$ and thus $CD=11.2065$~\cite{demvsar2006statistical}.

\begin{table}[ht]
\renewcommand{\arraystretch}{1.2}
%\scriptsize
\centering
%\captionsetup{font=scriptsize,justification=centering}
\caption{Summary of the Friedman Statistics $F_F(\mathcal{K}=20,\mathcal{N}=7)$ and the	Critical Value in Terms of Each Evaluation Metric ($\mathcal{K}$: \# Comparing Algorithms; $\mathcal{N}$: \# Data Sets).}
\begin{tabular}{llc}
	\toprule
	Metric &$F_F$&Critical Value ($\alpha = 0.05$)\\
	\toprule
	AP&	10.1378	
	&\multirow{4}{*}{1.6785}\\
	%Subset Accuracy&5.584&\\
	AUC&11.1020&\\
	NDCG&11.3810&\\
	RR&15.2027&\\
	%\bottomrule
	\hline
\end{tabular}
\label{ffTest}
\end{table}

\begin{table*}
\renewcommand{\arraystretch}{1}
%\tiny
%\scriptsize
\centering		%\captionsetup{font=scriptsize,justification=centering}
\caption{Experimental results of each comparing algorithm in terms of AP, AUC, NDCG, and RR.}
\adjustbox{max width=\linewidth}{
	\begin{tabular}{lclclclclclclcl}
		\toprule
		\multirow{2}{*}{Comparing}&\multicolumn{14}{c}{AP}\\ \cline{2-15}
		algorithm & Personality-2018 && Flixsters && MovLens 10M &&	MovieLens 20M && MovieLens 25M &&	MovieLens-Latest& & MovieLens-Latest-V2 &  \\ \hline
		Precedence Mining & 0.19 & 12 & 0.15 & 14     & 0.06 & 20   & 0.03 & 20  & 0.10 & 20  & 0.01 & 20 & 0.08 & 20  \\
		CRS-Med           & 0.13 & 19  & 0.16 & 9      & 0.15 & 16   & 0.15 & 16  & 0.14 & 13  & 0.15 & 15 & 0.17 & 19  \\
		CRS-Max           & 0.15 & 14  & \textbf{0.20} & \textbf{1}      & 0.18 & 7.5  & 0.18 & 5   & \textbf{0.16} & \textbf{4.5} & \textbf{0.18} & \textbf{4}  & 0.21 & 6   \\
		ICRS1             & 0.23 & 8   & 0.19 & 2.5    & 0.15 & 16   & 0.15 & 16  & 0.13 & 16  & 0.15 & 15 & 0.19 & 16  \\
		ICRS2             & 0.14  & 16.5 & 0.08 & 18.5   & 0.16 & 12.5 & 0.16 & 12  & 0.14 & 13  & 0.15 & 15 & 0.20 & 12  \\
		ICRS3             & \textbf{0.26} & \textbf{1.5}  & 0.18 & 4.5  & 0.18 & 7.5  & 0.18 & 5   & \textbf{0.16} & \textbf{4.5} & 0.17 & 9  & 0.21 & 6   \\
		ICRS4             & \textbf{0.26} & \textbf{1.5}  & 0.18 & 4.5  & \textbf{0.19} & \textbf{2.5}  & 0.18 & 5   & \textbf{0.16} & \textbf{4.5} & \textbf{0.18} & \textbf{4}  & 0.21 & 6   \\
		ICRS5             & 0.22 & 10.5 & 0.19 & 2.5  & 0.15 & 16   & 0.15 & 16  & 0.13 & 16  & 0.15 & 15 & 0.20 & 12  \\
		ICRS6             & 0.16 & 13   & 0.08 & 18.5 & 0.15 & 16   & 0.16 & 12  & 0.13 & 16  & 0.15 & 15 & 0.19 & 16  \\
		ICRS7             & 0.24 & 5    & 0.16 & 9    & \textbf{0.19} & \textbf{2.5}  & 0.18 & 5   & \textbf{0.16} & \textbf{4.5} & \textbf{0.18} & \textbf{4}  & 0.21 & 6   \\
		ICRS8             & 0.24 & 5    & 0.15 & 14   & \textbf{0.19} & \textbf{2.5}  & \textbf{0.19} & \textbf{1}   & \textbf{0.16} & \textbf{4.5} & \textbf{0.18} & \textbf{4}  & 0.20 & 12  \\
		ICRS9             & 0.14 & 16.5 & 0.15 & 14   & 0.17 & 11   & 0.16 & 12  & 0.15 & 10  & 0.16 & 11 & 0.20 & 12  \\
		ICRS10            & 0.14 & 16.5 & 0.09 & 17   & 0.15 & 16   & 0.15 & 16  & 0.14 & 13  & 0.15 & 15 & 0.20 & 12  \\
		ICRS11            & 0.22 & 10.5 & 0.16 & 9    & 0.18 & 7.5  & 0.18 & 5   & \textbf{0.16} & \textbf{4.5} & \textbf{0.18} & \textbf{4}  & \textbf{0.22} & \textbf{1.5} \\
		ICRS12            & 0.23 & 8    & 0.15 & 14   & 0.18 & 7.5  & 0.18 & 5   & \textbf{0.16} & \textbf{4.5} & \textbf{0.18} & 4  & 0.21 & 6   \\
		ICRS13            & 0.14 & 16.5 & 0.15 & 14   & 0.18 & 7.5  & 0.17 & 9.5 & \textbf{0.16} & \textbf{4.5} & 0.17 & 9  & \textbf{0.22} & \textbf{1.5} \\
		ICRS14            & 0.07 & 20   & 0.05 & 20   & 0.14 & 19   & 0.14 & 19  & 0.11 & 19  & 0.13 & 19 & 0.19 & 16  \\
		ICRS15            & 0.24 & 5    & 0.16 & 9    & 0.18 & 7.5  & 0.17 & 9.5 & 0.15 & 10  & 0.17 & 9  & 0.21 & 6   \\
		ICRS16            & 0.23 & 8    & 0.16 & 9    & \textbf{0.19} & \textbf{2.5}  & 0.18 & 5   & 0.15 & 10  & \textbf{0.18} & \textbf{4}  & 0.21 & 6   \\
		ICRS17            & 0.25 & 3    & 0.17 & 6    & 0.16 & 12.5 & 0.15 & 16  & 0.12 & 18  & 0.15 & 15 & 0.18 & 18 \\
		\toprule
		%
		%%AUC
		\multirow{2}{*}{Comparing}&\multicolumn{14}{c}{AUC}\\ \cline{2-15}
		algorithm & Personality-2018 && Flixsters && MovLens 10M &&	MovieLens 20M && MovieLens 25M &&	MovieLens-Latest& & MovieLens-Latest-V2 &  \\ \hline
		Precedence Mining & 0.92 & 2   & \textbf{0.96} & \textbf{1}    & 0.64 & 20  & 0.90 & 14   & \textbf{0.98} & \textbf{1}    & 0.65 & 20  & \textbf{0.92} & \textbf{1}    \\
		CRS-Med           & 0.62 & 16  & 0.87 & 7.5  & 0.72 & 19  & 0.75 & 20   & 0.77 & 20   & 0.75 & 19  & 0.67 & 20   \\
		CRS-Max           & 0.57 & 18  & 0.83 & 13.5 & 0.84 & 14  & 0.86 & 16   & 0.88 & 16   & 0.87 & 15  & 0.80 & 18   \\
		ICRS1             & \textbf{0.93} & \textbf{1}   & 0.91 & 2.5  & \textbf{0.90} & \textbf{3}   & \textbf{0.92} & \textbf{5.5}  & 0.93 & 6    & \textbf{0.93} & \textbf{4.5} & 0.86 & 6    \\
		ICRS2             & 0.67 & 15  & 0.60 & 17   & 0.83 & 15  & 0.87 & 15   & 0.89 & 15   & 0.88 & 14  & 0.84 & 14.5 \\
		ICRS3             & 0.91 & 4   & 0.88 & 5.5  & \textbf{0.90} & \textbf{3}   & \textbf{0.92} & \textbf{5.5}  & 0.94 & 2    & 0.92 & 11  & 0.86 & 6    \\
		ICRS4             & 0.91 & 4   & 0.89 & 4    & \textbf{0.90} & \textbf{3}   & \textbf{0.92} & \textbf{5.5}  & 0.92 & 11.5 & \textbf{0.93} & \textbf{4.5} & 0.85 & 11   \\
		ICRS5             & 0.91 & 4   & 0.91 & 2.5  & 0.89 & 9.5 & \textbf{0.92} & \textbf{5.5}  & 0.93 & 6    & 0.92 & 11  & 0.85 & 11   \\
		ICRS6             & 0.59 & 17  & 0.52 & 18   & 0.81 & 16  & 0.84 & 17   & 0.86 & 17   & 0.86 & 16  & 0.81 & 16   \\
		ICRS7             & 0.90 & 6   & 0.86 & 10   & \textbf{0.90} & \textbf{3}   & \textbf{0.92} & \textbf{5.5}  & 0.92 & 11.5 & \textbf{0.93} & \textbf{4.5} & 0.86 & 6    \\
		ICRS8             & 0.88 & 8.5 & 0.79 & 15   & 0.89 & 9.5 & \textbf{0.92} & \textbf{5.5}  & 0.92 & 11.5 & 0.92 & 11  & 0.84 & 14.5 \\
		ICRS9             & 0.87 & 11  & 0.87 & 7.5  & 0.89 & 9.5 & \textbf{0.92} & \textbf{5.5}  & 0.93 & 6    & \textbf{0.93} & \textbf{4.5} & 0.85 & 11   \\
		ICRS10            & 0.53 & 19  & 0.47 & 19   & 0.78 & 17  & 0.82 & 18.5 & 0.84 & 18   & 0.84 & 17  & 0.80 & 18   \\
		ICRS11            & 0.87 & 11  & 0.84 & 12   & 0.89 & 9.5 & 0.91 & 12   & 0.93 & 6    & \textbf{0.93} & \textbf{4.5} & 0.87 & 2.5  \\
		ICRS12            & 0.85 & 14  & 0.75 & 16   & 0.89 & 9.5 & 0.91 & 12   & 0.91 & 14   & 0.92 & 11  & 0.85 & 11   \\
		ICRS13            & 0.86 & 13  & 0.86 & 10   & \textbf{0.90} & \textbf{3}   & \textbf{0.92} & \textbf{5.5}  & 0.93 & 6    & \textbf{0.93} & \textbf{4.5} & 0.87 & 2.5  \\
		ICRS14            & 0.51 & 20  & 0.45 & 20   & 0.74 & 18  & 0.82 & 18.5 & 0.83 & 19   & 0.83 & 18  & 0.80 & 18   \\
		ICRS15            & 0.88 & 8.5 & 0.83 & 13.5 & 0.89 & 9.5 & \textbf{0.92} & \textbf{5.5}  & 0.92 & 11.5 & 0.92 & 11  & 0.85 & 11   \\
		ICRS16            & 0.87 & 11  & 0.86 & 10   & 0.89 & 9.5 & \textbf{0.92} & \textbf{5.5}  & 0.93 & 6    & \textbf{0.93} & \textbf{4.5} & 0.86 & 6    \\
		ICRS17            & 0.89 & 7   & 0.88 & 5.5  & 0.89 & 9.5 & 0.91 & 12   & 0.93 & 6    & \textbf{0.93} & \textbf{4.5} & 0.86 & 6    \\
		\toprule
		%
		%%NDCG
		\multirow{2}{*}{Comparing}&\multicolumn{14}{c}{NDCG}\\ \cline{2-15}
		algorithm & Personality-2018 && Flixsters && MovLens 10M &&	MovieLens 20M && MovieLens 25M &&	MovieLens-Latest& & MovieLens-Latest-V2 &  \\ \hline
		Precedence Mining & 0.66 & 12   & 0.53 & 13.5 & 0.16 & 20   & 0.37 & 20   & 0.48 & 19.5 & 0.27 & 20   & 0.41 & 20   \\
		CRS-Med           & 0.58 & 18   & 0.55 & 7    & 0.54 & 17.5 & 0.53 & 18   & 0.52 & 15.5 & 0.53 & 17.5 & 0.48 & 19   \\
		CRS-Max           & 0.59 & 16.5 & \textbf{0.61} & \textbf{1}    & \textbf{0.59} & \textbf{4}    & \textbf{0.58} & \textbf{3.5}  & \textbf{0.56} & \textbf{2.5}  & \textbf{0.58} & \textbf{1}    & 0.54 & 7    \\
		ICRS1             & 0.70 & 5.5  & 0.57 & 3    & 0.56 & 13.5 & 0.55 & 13.5 & 0.53 & 12.5 & 0.55 & 12   & 0.52 & 14.5 \\
		ICRS2             & 0.59 & 16.5 & 0.44 & 17   & 0.55 & 15.5 & 0.55 & 13.5 & 0.52 & 15.5 & 0.53 & 17.5 & 0.53 & 12   \\
		ICRS3             & \textbf{0.72} & \textbf{1.5}  & 0.56 & 5    & \textbf{0.59} & \textbf{4}    & 0.57 & 8    & 0.55 & 7.5  & 0.57 & 5.5  & 0.54 & 7    \\
		ICRS4             & \textbf{0.72} & \textbf{1.5}  & 0.57 & 3    & \textbf{0.59} & \textbf{4}    & \textbf{0.58} & \textbf{3.5}  & 0.55 & 7.5  & 0.57 & 5.5  & 0.54 & 7    \\
		ICRS5             & 0.69 & 9.5  & 0.57 & 3    & 0.56 & 13.5 & 0.55 & 13.5 & 0.53 & 12.5 & 0.54 & 14.5 & 0.53 & 12   \\
		ICRS6             & 0.60 & 15   & 0.43 & 18.5 & 0.55 & 15.5 & 0.55 & 13.5 & 0.51 & 18   & 0.54 & 14.5 & 0.51 & 17   \\
		ICRS7             & 0.70 & 5.5  & 0.55 & 7    & \textbf{0.59} & \textbf{4}    & \textbf{0.58} & \textbf{3.5}  & 0.55 & 7.5  & 0.57 & 5.5  & 0.54 & 7    \\
		ICRS8             & 0.70 & 5.5  & 0.52 & 15.5 & \textbf{0.59} & \textbf{4}    & \textbf{0.58} & \textbf{3.5}  & 0.55 & 7.5  & 0.57 & 5.5  & 0.53 & 12   \\
		ICRS9             & 0.62 & 13.5 & 0.54 & 10.5 & 0.58 & 9.5  & 0.55 & 13.5 & 0.55 & 7.5  & 0.56 & 10.5 & 0.54 & 7    \\
		ICRS10            & 0.57 & 19   & 0.43 & 18.5 & 0.54 & 17.5 & 0.54 & 17   & 0.52 & 15.5 & 0.54 & 14.5 & 0.52 & 14.5 \\
		ICRS11            & 0.69 & 9.5  & 0.54 & 10.5 & \textbf{0.59} & \textbf{4}    & \textbf{0.58} & \textbf{3.5}  & \textbf{0.56} & \textbf{2.5}  & 0.57 & 5.5  & \textbf{0.56} & \textbf{1}    \\
		ICRS12            & 0.68 & 11   & 0.52 & 15.5 & \textbf{0.59} & \textbf{4}    & \textbf{0.58} & \textbf{3.5}  & \textbf{0.56} & \textbf{2.5}  & 0.57 & 5.5  & 0.54 & 7    \\
		ICRS13            & 0.62 & 13.5 & 0.53 & 13.5 & 0.58 & 9.5  & 0.57 & 8    & \textbf{0.56} & \textbf{2.5}  & 0.57 & 5.5  & 0.55 & 2.5  \\
		ICRS14            & 0.50 & 20   & 0.38 & 20   & 0.51 & 19   & 0.51 & 19   & 0.48 & 19.5 & 0.50 & 19   & 0.51 & 17   \\
		ICRS15            & 0.70 & 5.5  & 0.54 & 10.5 & 0.58 & 9.5  & 0.56 & 10   & 0.54 & 11   & 0.56 & 10.5 & 0.54 & 7    \\
		ICRS16            & 0.70 & 5.5  & 0.54 & 10.5 & 0.58 & 9.5  & 0.57 & 8    & 0.55 & 7.5  & 0.57 & 5.5  & 0.55 & 2.5  \\
		ICRS17            & 0.70 & 5.5  & 0.55 & 7    & 0.57 & 12   & 0.55 & 13.5 & 0.52 & 15.5 & 0.54 & 14.5 & 0.51 & 17  \\
		\toprule
		%
		%%%%%% RR 			
		\multirow{2}{*}{Comparing}&\multicolumn{14}{c}{RR}\\ \cline{2-15}
		algorithm & Personality-2018 && Flixsters && MovLens 10M &&	MovieLens 20M && MovieLens 25M &&	MovieLens-Latest& & MovieLens-Latest-V2 &  \\ \hline
		Precedence Mining & 0.82 & 12   & 0.19 & 20   & 0.31 & 20   & 0.24 & 20   & 0.34 & 20   & 0.22 & 20   & 0.25 & 20   \\
		CRS-Med           & 0.78 & 14   & 0.60 & 2    & 0.59 & 14.5 & 0.64 & 9    & 0.59 & 5.5  & 0.62 & 6.5  & 0.55 & 18.5 \\
		CRS-Max           & \textbf{0.97} & \textbf{1}    & \textbf{0.92} & \textbf{1}    & \textbf{0.76} & \textbf{1}    & \textbf{0.76} & \textbf{1}    & \textbf{0.71} & \textbf{1}    & \textbf{0.75} & \textbf{1}    & \textbf{0.72} & \textbf{1}    \\
		ICRS1             & 0.85 & 10.5 & 0.51 & 4.5  & 0.60 & 12.5 & 0.60 & 14   & 0.56 & 12   & 0.59 & 12   & 0.60 & 13   \\
		ICRS2             & 0.51 & 19   & 0.37 & 16   & 0.53 & 18   & 0.57 & 16.5 & 0.49 & 17.5 & 0.51 & 18   & 0.58 & 15.5 \\
		ICRS3             & 0.91 & 2.5  & 0.47 & 7.5  & 0.65 & 4    & 0.65 & 6.5  & 0.59 & 5.5  & 0.62 & 6.5  & 0.63 & 5.5  \\
		ICRS4             & 0.91 & 2.5  & 0.47 & 7.5  & 0.64 & 6.5  & 0.66 & 4    & 0.61 & 3    & 0.62 & 6.5  & 0.62 & 7.5  \\
		ICRS5             & 0.85 & 10.5 & 0.52 & 3    & 0.59 & 14.5 & 0.60 & 14   & 0.57 & 10   & 0.59 & 12   & 0.61 & 10.5 \\
		ICRS6             & 0.62 & 17   & 0.32 & 18   & 0.54 & 17   & 0.60 & 14   & 0.49 & 17.5 & 0.53 & 17   & 0.56 & 17   \\
		ICRS7             & 0.86 & 7.5  & 0.45 & 10.5 & 0.63 & 8.5  & 0.67 & 2.5  & 0.56 & 12   & 0.61 & 9.5  & 0.61 & 10.5 \\
		ICRS8             & 0.86 & 7.5  & 0.42 & 15   & 0.65 & 4    & 0.65 & 6.5  & 0.59 & 5.5  & 0.62 & 6.5  & 0.62 & 7.5  \\
		ICRS9             & 0.69 & 15   & 0.48 & 6    & 0.60 & 12.5 & 0.56 & 18   & 0.54 & 16   & 0.56 & 16   & 0.61 & 10.5 \\
		ICRS10            & 0.57 & 18   & 0.36 & 17   & 0.56 & 16   & 0.57 & 16.5 & 0.55 & 14.5 & 0.57 & 14.5 & 0.58 & 15.5 \\
		ICRS11            & 0.86 & 7.5  & 0.44 & 12.5 & 0.64 & 6.5  & 0.67 & 2.5  & 0.58 & 8.5  & 0.63 & 3    & 0.67 & 2    \\
		ICRS12			  & 0.79 & 13   & 0.43 & 14   & 0.63 & 8.5  & 0.65 & 6.5  & 0.63 & 2    & 0.61 & 9.5  & 0.61 & 10.5 \\
		ICRS13            & 0.65 & 16   & 0.44 & 12.5 & 0.65 & 4    & 0.62 & 10.5 & 0.59 & 5.5  & 0.63 & 3    & 0.63 & 5.5  \\
		ICRS14            & 0.40 & 20   & 0.29 & 19   & 0.47 & 19   & 0.50 & 19   & 0.43 & 19   & 0.46 & 19   & 0.55 & 18.5 \\
		ICRS15            & 0.88 & 5    & 0.46 & 9    & 0.61 & 10.5 & 0.61 & 12   & 0.56 & 12   & 0.59 & 12   & 0.64 & 3.5  \\
		ICRS16            & 0.89 & 4    & 0.45 & 10.5 & 0.67 & 2    & 0.65 & 6.5  & 0.58 & 8.5  & 0.63 & 3    & 0.64 & 3.5  \\
		ICRS17            & 0.86 & 7.5  & 0.51 & 4.5  & 0.61 & 10.5 & 0.62 & 10.5 & 0.55 & 14.5 & 0.57 & 14.5 & 0.59 & 14	\\
		\toprule
	\end{tabular}
}
\label{tab:AP_AUC_NDCG_RR}
\end{table*}

\begin{figure*}
\centering
\begin{subfigure}{\textwidth}
	\centering
	\includegraphics[width=\textwidth,height=1.8in]{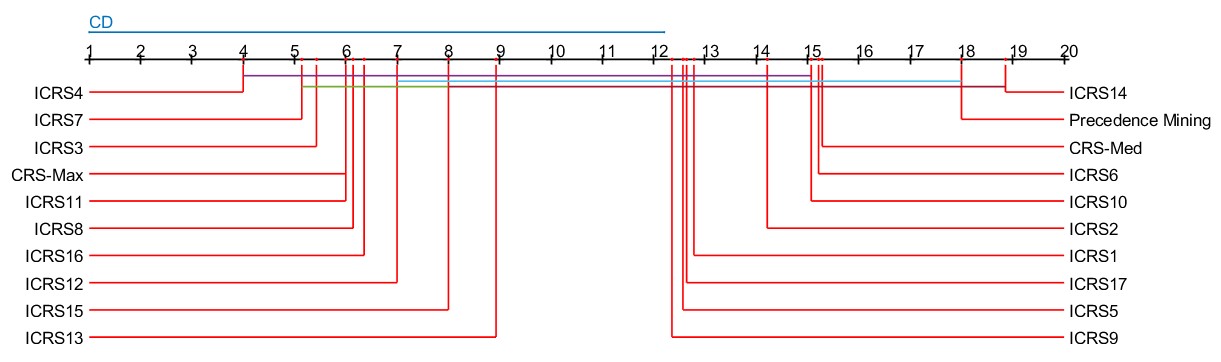}
	\caption{AP}
	\label{fig:sub2}
\end{subfigure}\\
\begin{subfigure}{\textwidth}
	\centering
	\includegraphics[width=\textwidth,height=1.8in]{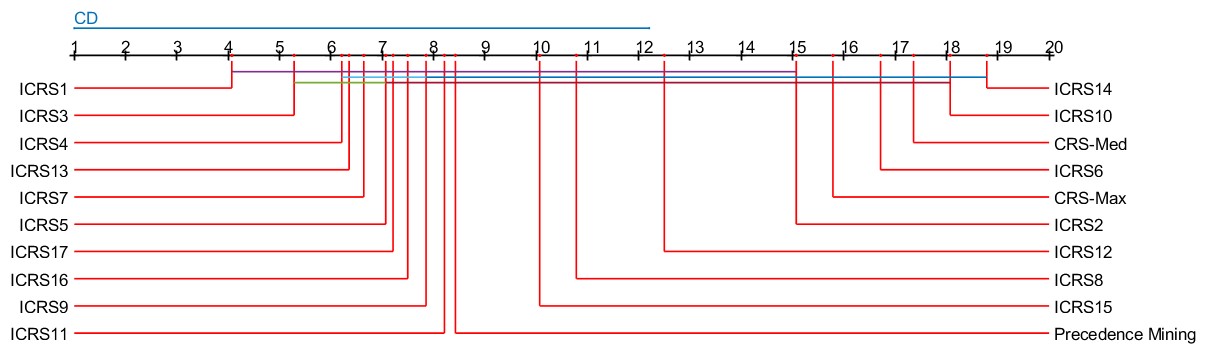}
	\caption{AUC}
	\label{fig:sub4}
\end{subfigure}\\
\begin{subfigure}{\textwidth}
	\centering
	\includegraphics[width=\textwidth,height=1.8in]{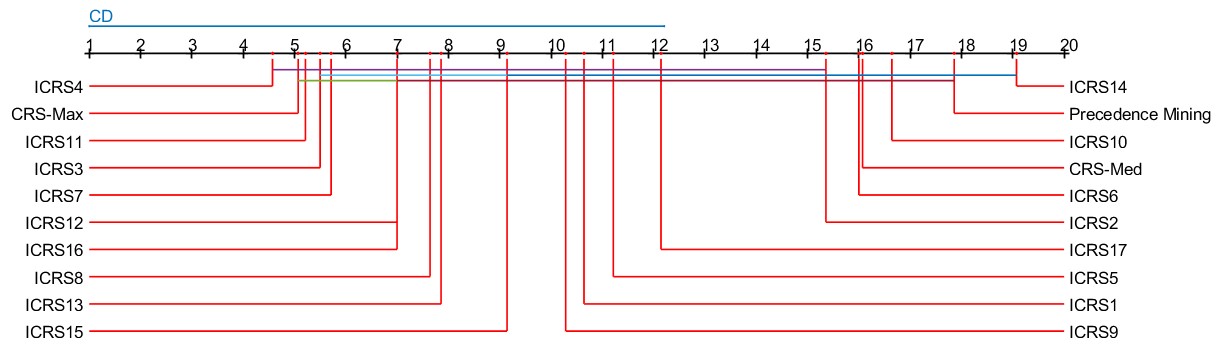}
	\caption{NDCG}
	\label{fig:sub5}%
\end{subfigure}\\
\begin{subfigure}{\textwidth}
	\centering
	\includegraphics[width=\textwidth,height=1.8in]{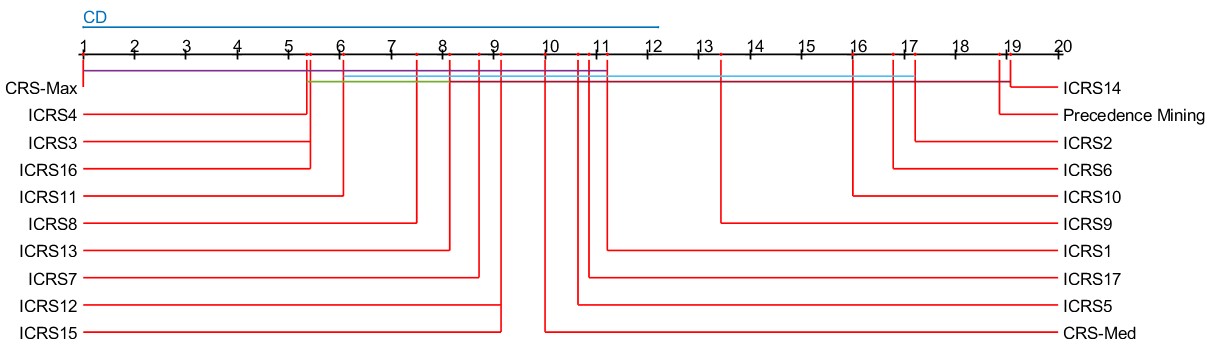}
	\caption{RR}
	\label{fig:sub6}
\end{subfigure}
\caption{CD diagrams of the comparing algorithms under each evaluation criterion.}
\label{CDDiagram}
\end{figure*}

Figure~\ref{CDDiagram} gives the CD diagrams~\cite{demvsar2006statistical} for each evaluation criterion, where the average rank of each comparing algorithm is marked along the axis (lower ranks to the left). It can be seen from the Figure~\ref{CDDiagram} that the proposed methods achieve  better performance than \textit{CRS-Med} and \textit{Precedence Mining} models over most of the evaluation metrics. We can also observe that the proposed approaches achieve similar performance to \textit{CRS-Max} or even derive a better rank in most cases, especially the median-based inductive approaches (\emph{ICRS3},\emph{ICRS7}, \emph{ICRS11}, and  \emph{ICRS15}) and mean-based inductive approaches (\emph{ICRS4},\emph{ICRS8}, \emph{ICRS12}, and \emph{ICRS16}). We have also observed similar results with Maximum  strategy based methods (\emph{ICRS1}, \emph{ICRS5}, \emph{ICRS9}, \emph{ICRS13}, and \emph{ICRS17}), whereas Minimum strategy based approaches (\emph{ICRS2}, \emph{ICRS6}, \emph{ICRS10}, and  \emph{ICRS14}) performing poorly among the seventeen proposed approaches.  
{%A detailed analysis reveals that \emph{Median} and \emph{Mean} based approaches capture the true precedence relations of a new item with respect to a user profile, as compared to \emph{Minimum} strategy. 
This comprehensive analysis reveals that Median and Mean-based approaches capture the true precedence relations of a new item with respect to a user profile compared to the Minimum strategy. The reason could be  that sometimes there is a higher chance of a user consuming items that are not of his regular interest but due to other users' influence (like family, friends, etc.)  or  situational context . These items do not follow good precedence relations with users' actual interests. Therefore, it is evident that there is a higher chance of minimum-based strategies capturing such precedence relations and that do not represent the allure of a new item concerning the user profile.
%This could be due to the fact that there is higher chance of a user consuming items that are not of his regular interest but due to the situational context or influence of other people (like family, friends, etc.). These items do not follow good precedence relations with the user's actual interests. Therefore, it is obvious that there is higher chance of minimum based strategies capture such precedence relations and that actually do not represent the allure of a new item concerning the user profile. 
In other words, there may be some noisy/outlier points in the user profile, and minimum-based strategies are more attractive to these points and therefore not suitable to measure (non)conformity. Though the same could be valid with the \emph{Maximum} based strategies, it is more likely that even a single item in the profile can influence to consume another item that follows higher precedence relations. 
%In simple words, there may be some noisy/outlier points in the user profile and minimum based strategies are more attractive to these points and therefore not suitable to measure (non)conformity. Though the same could be true with the \emph{Maximum} strategy based approaches, it is more likely that even a single item in the profile can influence to consume another item that follow higher precedence relations. 
For example, a Deep Learning course may have a higher precedence relation with a Machine Learning course, and that could be an influencing factor for a student to opt for a Deep Learning course irrespective of other courses in the student profile. Experiment results corroborate our claims.}
%For example, Deep Learning course may have a higher precedence relations with Machine Learning course and that could be a influencing factor for a student to opt for Deep Learning course irrespective of other courses in the student profile. Experiment results corroborate our claims.   }

%are always giving comparable performing as compared well when the dataset size is slightly large, that is, for the datasets \emph{ML 10M, ML 20M,} and  \emph{ML-latest}.  More specifically,  \emph{ICRS3, ICRS4, ICRS7,} and  \emph{ICRS8} outperforms other approaches over these datasets. We notice that \emph{ICRS10}-\emph{ICRS13}  are not performing well overall as expected from its efficiency graph in Figure~\ref{fig:ICRS_err2}. 

%It can also be seen that the CRS-max and CRS-med gives marginally better performance as compared to the proposed methods over \textit{Lastfm}, \textit{Foursquare}, and \textit{Gowalla} datasets.  As mentioned previously, these datasets violated our assumption about the items' precedence relations. Hence, we can conclude that the inductive variants outperform the non-inductive variants when the dataset size is considerably large and meets the model assumptions. 

\begin{figure*}
\centering
\begin{subfigure}{\textwidth}
\centering
\includegraphics[width=\textwidth,height=2.5in]{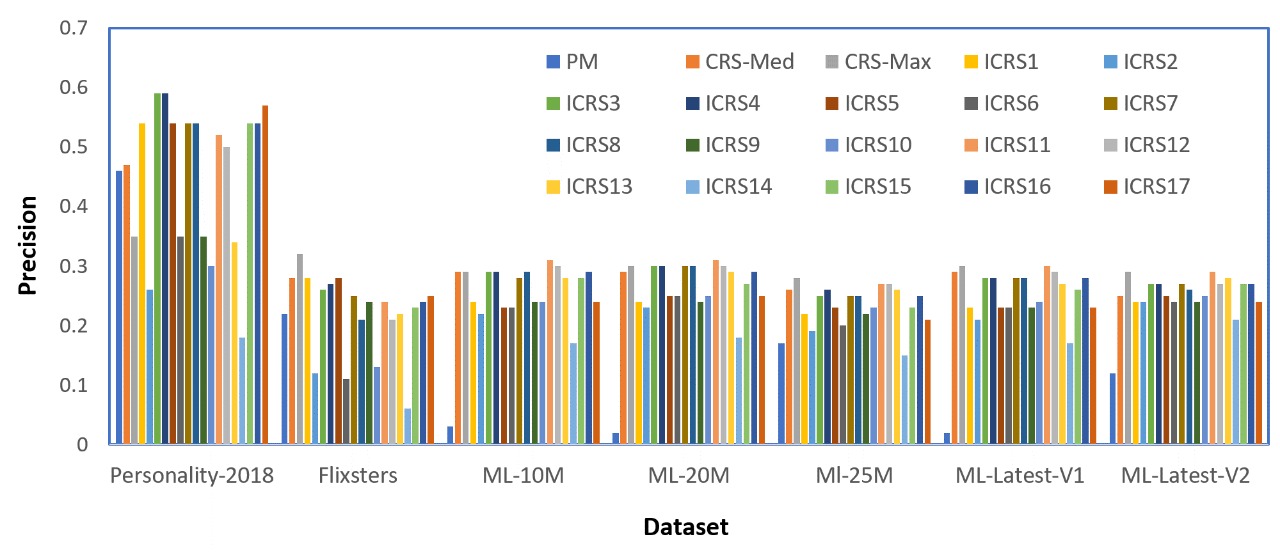}
\caption{precision@10}
\label{fig:ICRS_precI}
\end{subfigure}\\
\vspace{0.1in}
\begin{subfigure}{\textwidth}
\centering
\includegraphics[width=\textwidth,height=2.5in]{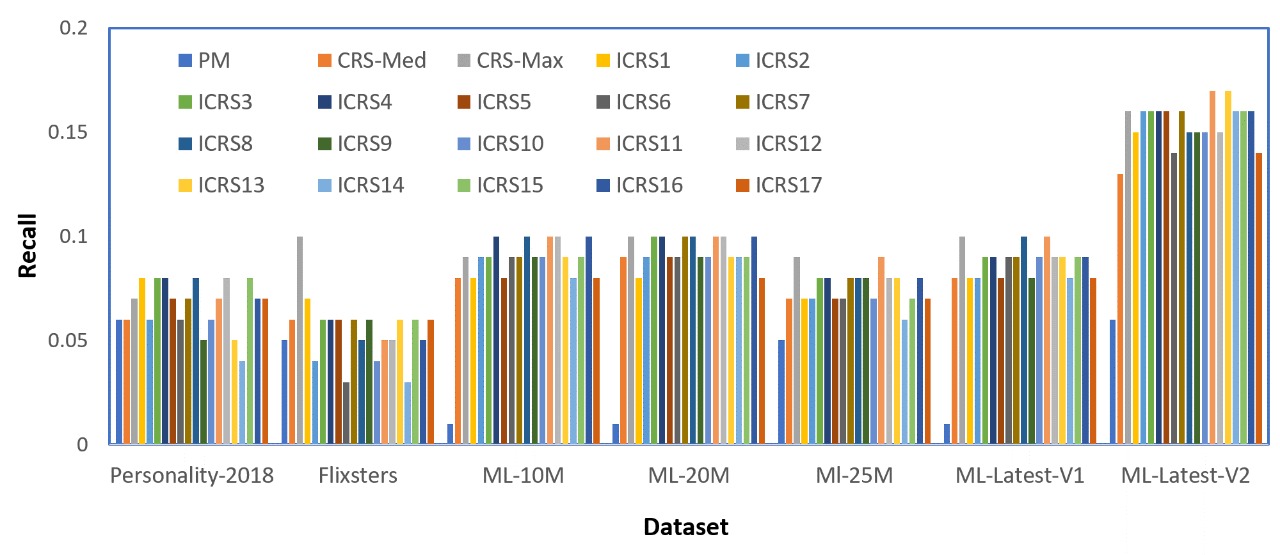}
\caption{recall@10}
\label{fig:ICRS_recalI}
\end{subfigure}\\
\vspace{0.1in}
\begin{subfigure}{\textwidth}
\centering
\includegraphics[width=\textwidth,height=2.5in]{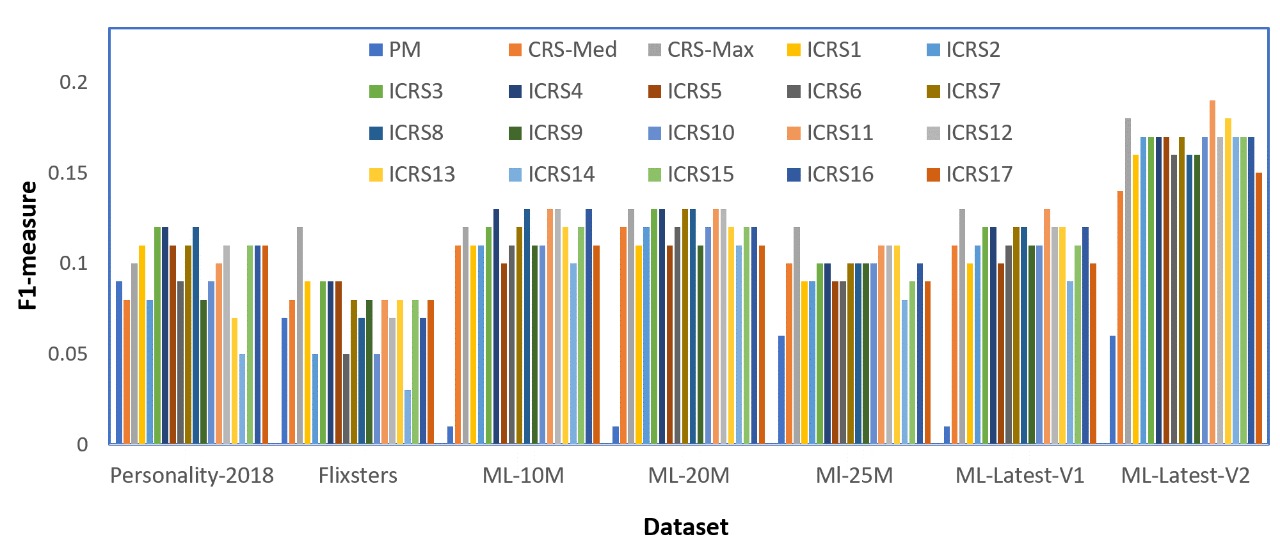}
\caption{F$_1$@10}
\label{fig:ICRS_fI}%
\end{subfigure}\\

\caption{ Performance of each comparing algorithm in terms of top-k recommendation metrics.}
\label{fig:ICRS_PRF1}
\end{figure*}

\vspace{0.2cm}
%\noindent \textbf{Comparative Analysis in Terms of Top-K recommendation Metrics: }
In the second set of experiments, we compare the performance in terms of top-$k$ recommendation measures, namely  \emph{precision@10}, \emph{recall@10} and \emph{F1@10} (Figure~\ref{fig:ICRS_PRF1}). We observe the similar results with varying the number of recommendations. It can be observed from the figures that conformal approaches methods outperform the underlying precedence mining model. Furthermore, findings reveal that inductive variants are comparable with the  \emph{CRS-max} and  \emph{CRS-med}. %\VK{However, the set of inductive algorithms, \emph{ICRS10-13}, exhibit poor performance as compared to other methods.}
Finally, we  compared the execution time (in milliseconds) of the different approaches. It can be seen from the Figure~\ref{fig:ICRS_time} that the inductive conformal recommender systems are much faster than traditional conformal recommender systems and better than the precedence mining model. Altogether, the results corroborate our claim that the inductive variant achieves a similar level of accuracy compared to its counterparts but significantly reduces the execution time.

\begin{figure*}[!htbp]
\centering
\includegraphics[width=\textwidth, height = 2.5in]{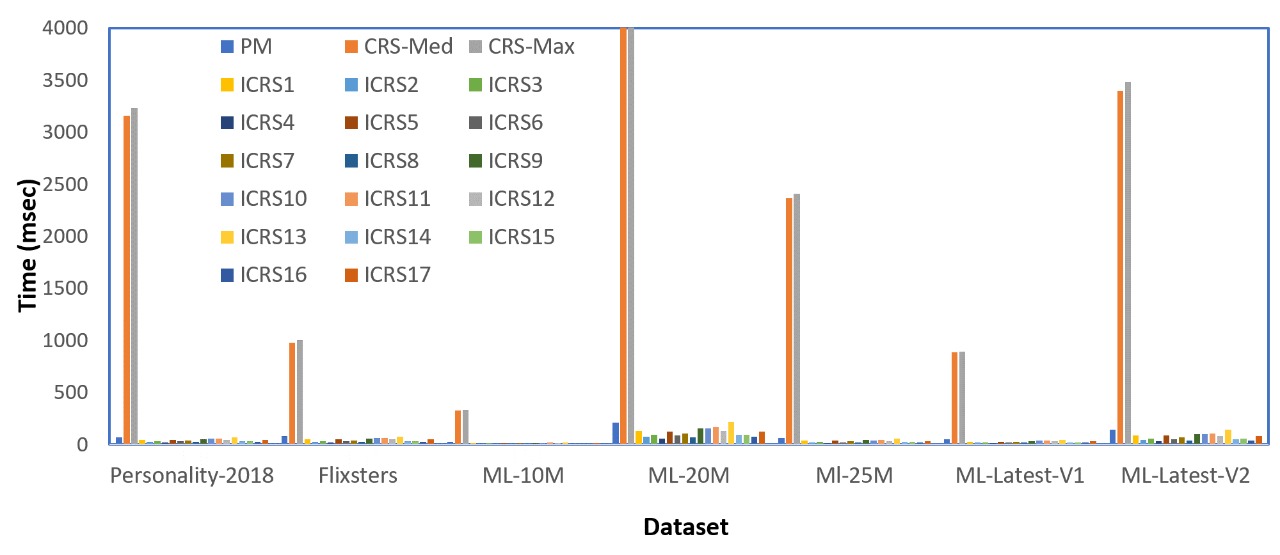}
\caption{ Execution time comparison for different datasets}
\label{fig:ICRS_time}
\end{figure*}

%\vspace{-0.15cm}
\section{Conclusions and Future Work}
\label{sec:conclusion}
In this paper, we propose an inductive variants of the conformal recommender system that complements the recommendation by quantifying the (un)reliability in predictions. One natural limitation with the existing transductive variants is the computation time that prevents their applicability in the time constraint domains. We address this limitation and propose an inductive variant that maintains the same moderate level of predictive accuracy but reduces the computation time to a large extent. Our conformal approach exemplifies confidence in terms of the bounds on the error.  Conformity/nonconformity measures are key component of any conformal recommendation framework, and the prediction accuracy largely depends on how well these measures are defined. In this work, we examined sevneteen different (non)conformity measures using the precedence relations among objects. We theoretically proved that the proposed (non)conformity measures adhere to the principle of validity under certain assumptions. Further, we emphasized our theoretical results with an empirical demonstration. %The paper also establishes that inductive methods take much lesser time as compared to other methods.
Rigorous experiments on several real-world datasets demonstrated that the inductive conformal recommendation algorithms outperform the precedence mining based recommender system and non-inductive methods in terms of execution time.  We observed that a few of the inductive variants outperforming the other approaches in terms of other crucial measures of the recommender system when the basic assumptions of the model are satisfied.  

The current proposal sets a lot of scope for future research.  Attaining the notion of confidence in different recommendation models by determining suitable (non)conformity measures is one of the exacting directions for enthusiastic researchers.  Investigating the conformal prediction for group recommender systems is a direction worth studying.  Exploring the conformal approach for different matrix factorization-based methods is another exciting direction to pursue.

% \newpage
%\bibliographystyle{plain}

\section*{Acknowledgements}
Venkateswara Rao Kagita is supported by the NITW-RSM grant, NIT Warangal. Vikas Kumar is supported by the Start-up Research Grant (SRG) under grant number \linebreak SRG/2021/001931 and the Faculty Research Programme Grant, University of Delhi under grant number IoE/2021/12/FRP. We would also like to thank the anonymous reviewers whose comments/suggestions helped improve and clarify this manuscript to a large extent.

\bibliographystyle{unsrt}
\bibliography{ecai}

%\iffalse	
%	
\newpage
\appendix
\section{Formal defintions of (non)conformity }
We give the formal definitions of the (non)conformity measures described in Section 3 as follows. 

\vspace{-0.5cm}
\begin{align*}
\noindent
CM2(o_h) &= \underset{o_i\in O_j^t}{minimum}~  PC(o_i, o_h).\\
CM3(o_h) &= \underset{o_i\in O_j^t}{median}~ PC(o_i, o_h).\\
CM4(o_h) &= \underset{o_i\in O_j^t}{mean}~ PC(o_i, o_h).\\
CM5(o_h) &= \underset{o_i\in O_j^t}{maximum}~  PC(o_i, o_h).\\
\noalign{\vskip1pt}	
CM6(o_h) &= \underset{o_i\in O_j^t}{ minimum} ~PP(o_h\mid o_i).\\
CM7(o_h) &= \underset{o_i\in O_j^t}{ median}  ~PP(o_h\mid o_i).\\
CM8(o_h) &= \underset{o_i\in O_j^t}{ mean}  ~PP(o_h\mid o_i).\\
CM9(o_h) &= \underset{o_i\in O_j^t}{ maximum} ~ PP(o_h\mid o_i).\\
\vspace{1cm}
CM10(o_h) &= \underset{o_i\in O_j^t}{minimum}~  \frac{PC(o_i, o_h)}{Sup(o_i) -PC(o_h, o_i) }.\\
CM11(o_h) &= \underset{o_i\in O_j^t}{median}~ \frac{PC(o_i, o_h)}{Sup(o_i) -PC(o_h, o_i) }.\\
CM12(o_h) &= \underset{o_i\in O_j^t}{mean}~ \frac{PC(o_i, o_h)}{Sup(o_i) -PC(o_h, o_i) }.\\
CM13(o_h) &= \underset{o_i\in O_j^t}{maximum}~  \frac{PC(o_i, o_h)}{Sup(o_i) -PC(o_h, o_i) }.\\
\vspace{1cm}
NCM14(o_h) &= \underset{o_i\in O_j^t}{minimum}~  \frac{Sup(o_i) -
PC(o_i, o_h)}{n_u}.\\
NCM15(o_h) &= \underset{o_i\in O_j^t}{median}~ \frac{Sup(o_i) -
PC(o_i, o_h)}{n_u}.\\
NCM16(o_h) &= \underset{o_i\in O_j^t}{mean}~ \frac{Sup(o_i) -
PC(o_i, o_h)}{n_u}.\\
NCM17(o_h) &= \underset{o_i\in O_j^t}{maximum}~  \frac{Sup(o_i) -
PC(o_i, o_h)}{n_u}.
\end{align*}
%\fi 	
\end{document}